\documentclass[journal]{IEEEtranTIE}
\usepackage{graphicx}
\usepackage{cite}
\usepackage{picinpar}
\usepackage{amsmath}
\usepackage{amssymb}
\usepackage{amsthm}
\usepackage{enumitem}
\usepackage{url}
\usepackage[utf8]{inputenc}
\usepackage{colortbl}
\usepackage{soul}
\usepackage{multirow}
\usepackage{pifont}
\usepackage{color}
\usepackage{alltt}
\usepackage[hidelinks]{hyperref}
\usepackage{enumerate}
\usepackage{siunitx}
\usepackage{breakurl}
\usepackage{pbox}
\usepackage{lineno}
\usepackage[printonlyused,withpage,nolist]{acronym}
\usepackage{tabularx} 
\newcommand{\myfloatalign}{\centering} 
\DeclareOldFontCommand{\bf}{\normalfont\bfseries}{\mathbf}
\usepackage{booktabs}

\AtBeginDocument{%
	\setlength\abovedisplayskip{-1pt}
}

\newcommand{\hlcyan}[1]{{\sethlcolor{white}\hl{#1}}}
\newcommand{\hlgreen}[1]{{\sethlcolor{white}\hl{#1}}}

\newtheorem{theorem*}{Theorem}

\newlist{Properties}{enumerate}{2}
\setlist[Properties]{label=\textit{Property} \arabic*:,itemindent=*}

\begin{document}
\title{	Nonlinear Model Predictive Control of a Robotic Soft Esophagus}

\author{
	\vskip 1em
	
	Dipankar Bhattacharya,
	Ryman Hashem,
	Leo~K. Cheng,
	\\ and Weiliang Xu, \emph{Senior Member, IEEE}

	\thanks{
		This is the author's accepted manuscript of an article published in
		\textit{IEEE Transactions on Industrial Electronics},
		vol.~69, no.~10, pp.~10363--10373, Oct.~2022,
		doi:~\href{https://doi.org/10.1109/TIE.2021.3121755}{10.1109/TIE.2021.3121755}.
		Electronic publication date: 27~October~2021.
		This work was supported
		financially by Riddet Institute, Palmerston North, New Zealand,
		which is a national center of research excellence. (Corresponding author:
		Weiliang Xu.)
		
		D. Bhattacharya, R. Hashem and W. Xu are with the Department of Mechanical
		Engineering, University of Auckland, Auckland 1010, New Zealand
		(e-mail: dbha483@aucklanduni.ac.nz; ryman\_87@live.com; p.xu@auckland.ac.nz).

		L. K. Cheng is with Auckland Bioengineering Institute, University
		of Auckland, Auckland 1010, New Zealand (e-mail: l.cheng@auckland.ac.nz).
	}
}

\makeatletter
\def\@IEEEpubidpullup{6.5\baselineskip}
\makeatother
\IEEEpubid{%
\begin{minipage}[b]{\textwidth}
\footnotesize\raggedright
\textcopyright~2021 IEEE. Personal use of this material is permitted.
Permission from IEEE must be obtained for all other uses, in any current or future media,
including reprinting/republishing this material for advertising or promotional purposes,
creating new collective works, for resale or redistribution to servers or lists,
or reuse of any copyrighted component of this work in other works.\\
Accepted manuscript. Version of record:
\href{https://doi.org/10.1109/TIE.2021.3121755}{doi:10.1109/TIE.2021.3121755}.
\end{minipage}%
}

\maketitle

\begin{center}
\small\bfseries
\href{https://bhattner143.github.io/rosev2-dtsindyc.github.io/}{[Project Page]}
\quad
\href{https://doi.org/10.1109/TIE.2021.3121755}{[Paper]}
\quad
\href{https://github.com/bhattner143/SINDYc_MPC_RoSE_symmetric_peristaltic}{[Code]}
\end{center}
\vspace{0.35em}

\begin{abstract}
Strictures caused by esophageal cancer can narrow down the
esophageal lumen, leading to dysphagia. Palliation of dysphagia has driven the development of a Robotic Soft Esophagus (RoSE), which provides a novel in vitro platform for esophageal stent testing and food viscosity studies. In RoSE, peristaltic wave generation and control were done in an open-loop manner since the conduit lacked visibility and embedded sensing capability.  Hence, in this work, RoSE version 2.0 (RoSEv2.0) is designed with embedded Time Of Flight (TOF) and pressure sensors to measure conduit displacement and air pressure, respectively, for modeling and control. Model Predictive Control (MPC) of RoSEv2.0 is implemented to govern the peristalsis and air pressure profile autonomously. The implemented MPC used Sparse Identification Nonlinear Dynamics with Control (SINDYC) models to estimate the future states of ROSEv2.0. The dynamic models are discovered from the TOF and pressure sensor data. Peristalsis waves of speed 20 mm.s$ ^{-1} $, wavelength 75  mm, and amplitudes 5, 7.5 , and 10  mm were successfully generated by the MPC.  Additionally, RoSEv2.0 with the MPC was employed to perform stent migration testing with various food boluses consistencies. The major contribution claimed in this paper is the application of \ac{SINDYC}-based \ac{MPC} to solve the closed-loop control problem of RoSE for achieving desired peristaltic waves. 

\end{abstract}


\markboth{}{}

\definecolor{limegreen}{rgb}{0.2, 0.8, 0.2}
\definecolor{forestgreen}{rgb}{0.00, 0.31, 0.0}
\definecolor{greenhtml}{rgb}{0.0, 0.5, 0.0}
\definecolor{myred}{rgb}{0.81, 0.0, 0.0}

\begin{acronym}[TDMA]
	\acro{ABS}{Acrylonitrile Butadiene Styrene}
	\acro{AC}{Abdominal Compression}
	\acro{ADC}{Analog to Digital Converter}
	\acro{ANN}{Artificial Neural Networks}
	\acro{ASR}{Automatic Symbolic Regression}
	\acro{BIBPS}{Baseline of Intra-Bolus Presssure Signature}
	
	\acro{CAD}{Computer-Aided Design}
	\acro{CeG}{CNT embroidered Graphane}
	\acro{CFM}{Cost Function Minimization}
	\acro{CNT}{Carbon Nanotube}
	\acro{COF}{Chronic Outward Force}
	\acro{CPG}{Central Pattern Generator}
	\acro{CSV}{Comma-separated Values}
	
	\acro{DAC}{Digital to Analog Converter}
	\acro{DE}{Differential Equation}
	\acro{SPDT}{Single Pole Double Throw}
	\acro{DT}{Discrete-Time}
	\acro{DTDE}{Discrete Time Differential Equation}
	\acro{DTSINDYC}{Discrete Time SINDYC}
	\acro{DOF}{Degree of Freedom}
	
	\acro{EC}{Esophageal Cancer}
	\acro{FDA}{Food and Drug Administration}
	\acro{FEA}{Finite Element Analysis}
	\acro{FEM}{Finite Element Modeling}
	\acro{FM}{Functional Margin}
	\acro{FSP}{Force Sensing Potentiometer}
	
	\acro{GM}{Geometric Margin}
	\acro{GPN}{Graphene Porous Network}
	\acro{HF}{Hoop Force}
	\acro{HRM}{High Resolution Manometry}
	
	\acro{I2C}{Inter-Integrated Circuit}
	\acro{IBPS}{Intra-Bolus Presssure Signature}
	\acro{IDC}{Insulation-Displacement Connector}
	\acro{IDLE}{Integrated Development
		and Learning Environment}
	\acro{ILPS}{Intraluminal Presssure Signature}
	\acro{IID}{Independently and Identically Distributed}
	\acro{IoT}{Internet of Things}
	\acro{IW}{Indentation Wave}
	
	\acro{LASSO}{Least Absolute Selection And Shrinkage Operator}
	\acro{LS}{Least Squares}
	\acro{MHB}{Multivariable Harmonic Balance}
	\acro{MIBPS}{Maximum of Intra-Bolus Presssure Signature}
	\acro{ML}{Machine Learning}
	\acro{MPC}{Model Predictive Control}
	\acro{MRI}{Magnetic Resonance Imaging}
	\acro{MIMO}{Multiple Input-Multiple Output}
	
	\acro{NRMSE}{Normalized Root Mean Square Error}
	
	\acro{OSS}{Optimal State Selection}
	
	\acro{PCR}{Principal Component Regression}
	\acro{PDMS}{Liquid Polydimethylsiloxane}
	\acro{PILPS}{Peak of Intraluminal Presssure Signature}
	\acro{PID}{Proportional-Integral-Derivative}
	\acro{PMA}{Pneumatic Muscle Actuator}
	\acro{POD}{Principle Orthogonal Decomposition}
	\acro{PP}{Primary Peristalsis}
	\acro{PPW}{Primary Peristaltic Wave}
	\acro{PSD}{Power Spectral Density}
	\acro{QRoSE}{Quarter-Robotic Soft Esophagus}
	\acro{QCQP}{Quadratically Constrained Quadratic Programming} 
	\acro{QP}{Quadratic Programming}
	
	\acro{RRF}{Radial Resistive Force}
	\acro{RCT}{Randomized Controlled Trials}
	\acro{RF}{Radial Force}
	\acro{RMSE}{Root Mean Square Error}
	\acro{ROM}{Reduced-Order Modeling}
	\acro{RoSE}{Robotic Soft Esophagus}
	\acro{RoSEv2.0}{Robotic Soft Esophagus version 2.0}
	\acro{ROT}{Region of Transition}
	\acro{RTV}{Room-Temperature-Vulcanizing}
	
	\acro{SEMS}{Self-Expandable Metallic Stent}	
	\acro{SEPS}{Self-Expandable Plastic Stent}
	\acro{SI}{System Identification}
	\acro{SINDY}{Sparse Identification of Nonlinear Dynamics}
	\acro{SINDYC}{Sparse identification of Nonlinear Dynamics with Control}
	\acro{SL}{Supervised Learning}
	\acro{SMA}{Shape Memory Alloy}
	\acro{SNT}{Silver Nanotube}
	\acro{SP}{Secondary Peristalsis}
	\acro{SPI}{Serial Peripheral Interface}
	\acro{SPW}{Secondary Peristaltic Wave}
	\acro{SR}{Symbolic Regression}
	\acro{STLSR}{Sequentially Thresholded Least Squares Regression}
	\acro{STRR}{Sequentially Thresholded Ridge Regression}
	\acro{SV}{Singular Value}
	\acro{SVD}{Singular Value Decomposition}
	\acro{SVM}{Support Vector Machines}
	
	\acro{TIFF}{Trajectory Induced Frictional Force}
	\acro{TOF}{Time of Flight}
	\acro{TPCR}{Total Principle Component Regression}
	\acro{TR}{Trust Region}
	\acro{TRRM}{Trust-Region-Reflective Method}
	\acro{UL}{Unsupervised Learning}
	
	\acro{VPS}{Valve Pressure Sensor}
	
\end{acronym}	

\acresetall


\section{Introduction}
\IEEEpubidadjcol

Esophageal cancer can cause esophageal stricture, which distresses lumen patency, leading to dysphagia \cite{Garcia2010}. Esophageal strictures can be addressed by implanting an endoprosthetic stent in the esophagus, which can hold open the esophagus and provide relief to patients \cite{hanawa2009materials}. However, stent migration caused by the stent interaction with the continuous peristaltic waves, is a significant shortcoming of concern to the patients \cite{sharma2010role}. 

\hlcyan{The migration can be minimized by improving stent designs, but evidence to determine which stent design is better than the other is confined to stent analytical} \cite{hirdes2013vitro} \hlcyan{and numerical modeling} \cite {garbey2016esophageal} \hlcyan{studies because in vivo stent testing can raise major ethical concerns. Garbey et al.} \cite{garbey2016esophageal} \hlcyan{and Mozafari et al.} \cite{mozafari2018migration} \hlcyan{are two relevant numerical studies that have performed stent testing on numerical esophageal models. While the former research focused on stent migration, the latter identified stent characteristics that significantly impacted stent migration.}

A \ac{RoSE} \hlcyan{has been developed as an alternative platform for conducting stent testing under bolus swallow conditions} \cite{bhattacharya2020rose}. \hlcyan{Since RoSE has been designed to mimic the human swallow behavior, matching the attributes of the biological esophagus; thus, instead of actual patients, RoSE has been used to conduct the study on the various stent designs before implanting them in patients with malignant esophageal strictures} \cite{bhattacharya2020rose}. 

\hlgreen{The dynamic modeling of soft and continuum robots like the RoSE has been challenging and remains an active research topic} \cite{george2018control}. \hlgreen{For soft continuum robots' controller design, many sensors are typically required to measure the high dimensional states to provide the controller with state feedback. Hence, several authors have considered task-space closed-loop control that does not require a dynamic model and a large number of sensors and can provide a robust and simplistic solution to the control problem without any previous information of the robots' dynamics. While Lee et al.} \cite{lee2017nonparametric} \hlgreen{proposed a model-less control framework for a soft fluid-driven continuum robot to follow a defined path accurately, Li et al.} \cite{li2017model} \hlgreen{introduced an adaptive Kalman-based controller to estimate the Jacobian of a} \ac{PMA}-\hlgreen{driven continuum robot.} 

{To avail from the vast knowledge of prevailing automatic control theory, soft robotics dynamic modeling is required since the more information a controller has about the plant, the better the tracking result} \cite{braganza2007neural}. {Additionally, unlike model-free control, model-based control approaches also take the compliant nature of such robots into consideration, and the controller can be operated in a high-speed environment. There is a growing body of literature that recognizes the importance of model-based control techniques, such as optimal control} \cite{dullerud2013course} and \ac{MPC} \cite{mayne2000constrained} {because they are capable of routinely taking actual plant constraints into account in time.}

\ac{MPC} is an advanced and well-known technique of controlling highly nonlinear processes with constraints \cite{mayne2000constrained,clarke1987generalizeda}. 
Based on the past observations, MPC employs the concept of receding horizon (forward shifting of prediction horizon) to anticipate potential disturbances and thereby to generate control sequence that is a proposition among possible alternatives that can manage disturbances well at the next time step \cite{mayne2000constrained,garcia1989model}. {Unlike traditional control theory, MPC is effective if there is a process dead-time or if the setpoint/reference trajectory is well defined ahead of time. Additionally, it provides the flexibility to formulate and modify the objective function for optimal control. }

Since the $ 1980 $s, there is a dramatic increase in the application of \ac{MPC} in industries \cite{garcia1989model} like power systems \cite{golbert2004model}, and robotics \cite{hou2020underwater,best2016new}, because \ac{MPC} can model and thereby predict nonlinear and non-minimum phase dynamics systems very accurately. Advanced control techniques like the \ac{MPC} relies on developing a suitable plant model for optimum performance. Soft robots like the \ac{RoSE} \cite{Dirven2014} and soft robotic gastric simulator \cite{dang2020sogut} manufactured by soft elastomers like silicone rubber, cannot be modeled from the perspective of first principles due to their complex, continuous, and highly compliant intrinsic deformation, which leads to infinite degrees of freedom \cite{rus2015}. 
%
%
%
With the increasing complexity of the soft robots and rigid performance specifications, the researchers have revealed that \ac{MPC} implementation with improved predictive model can significantly enhance the control performance of complex nonlinear systems \cite{andersen1992evaluating,best2016new}. 

Data-driven techniques are not a panacea for all engineering problems, but if used correctly, it provides a way to leverage the collected data from experiments maximally.  
The \ac{SINDYC} approach \cite{brunton2016discovering, brunton2016sparse},  has succeeded in addressing  challenges in the data-driven modeling of \ac{RoSE} \cite{bhattacharya2019sparse}. Unlike traditional black-box data-driven modeling techniques, SINDYC has demonstrated the power to identify nonlinear dynamics from data captured from a nonlinear system. The method has been applied and validated with several dynamical systems, such as the damped harmonic oscillator with linear or cubic dynamics, the Lorenz system, and identifying bifurcations and normal forms  \cite{brunton2016discovering}. Fields such as soft robotics, where there is the absence of governing equations, may benefit from methods like the SINDYC.
Bhattacharya et al. has proven that the \ac{SINDYC} algorithm promotes sparsity by identifying a few linear and nonlinear terms governing the dynamics of the \ac{RoSE} \cite{bhattacharya2019sparse}. Additionally, unlike other \ac{ML} techniques like the \ac{ANN}, \ac{SINDYC} demands less training data, less execution time, and avoids overfitting \cite{kaiser2018sparse}. Integrating the \ac{RoSE} \ac{SINDYC} model in the \ac{MPC} framework can take advantage of the underlying dynamics in achieving the prescribed peristalsis profile for stent testing. 


The primary aim of this paper is to actuate the \ac{RoSE} conduit with pre-defined peristaltic wave trajectories by implementing \ac{SINDYC}-based \ac{MPC} framework. The fulfillment of the aim will make \ac{RoSE} capable of generating various peristaltic wave shapes, which will further enhance the in vitro stent migration testing potential of \ac{RoSE}. Five significant steps were followed to implement the \ac{MPC} for stent migration testing. First, \ac{RoSEv2.0} was developed with an embedded array of \ac{TOF} sensors to address the issue of limited visibility inside the \ac{RoSEv2.0} conduit. Second, \ac{RoSEv2.0} was actuated with various pressure varying trajectories and data from \acp{TOF} and \acp{VPS} were collected. Third, discrete-time formulation of \ac{SINDYC} referred to as \ac{DTSINDYC} was applied to the collected data to discover \ac{DTDE} models of \ac{RoSEv2.0}. Fourth, nonlinear \ac{MPC} using the \ac{DTDE} models were designed and implemented to control the peristaltic motion of the \ac{RoSEv2.0} conduit. Last, stent migration data for a candidate stent were recorded for various peristalsis speed and bolus swallow conditions.

The research makes several noteworthy contributions to the field of soft robotics. 
This research has solved the sensing issue for modeling and control of \ac{RoSE} by introducing \ac{RoSEv2.0} (improved version of \ac{RoSE}) with \ac{TOF} sensors.
Additionally, to the best of our knowledge, the \ac{SINDYC}-based \ac{MPC} framework presented in this research has been applied for the first time in a soft robot. 
The closed-loop control in the form of MPC has been implemented for the first time in RoSEv2.0. \hlgreen{Overall, the contribution claimed in this paper is the application of SINDYC-MPC to solve the closed-loop control problem of RoSE for peristaltic wave tracking. }

\hlgreen{Additionally, in this research, the presented modeling and control methodology is not a robot or sensor specific, and it can be extended to any soft-robotic system with any sensor provided input-output datasets can be generated. Hence to verify the versatility of this approach, MPC of RoSEv2.0 controlling the chamber air pressure was implemented with VPS by following the same implementation methodology described with the TOF sensors.}

\section{Problem formulation}\label{sec_problem_formulation}

The foundation of supervised \ac{ML} techniques like the \ac{SINDYC} is qualitative and quantitative data collection. To a greater extent, the quality of the training data determines the \ac{SINDYC} modeling generalization when applied to any soft robot like \ac{RoSE} \cite{Dirven2014}. 
The updated \ac{RoSE} (\ac{RoSEv2.0}) has $ 12 $ identical layers ($ L_{1} $ to $ L_{12} $) such that each layer contains four air pressure chambers surrounding the conduit axis symmetrically. Each horizontal whorl of four chambers at each layer is connected to an electro-pneumatic pressure valve (ITV-0030-3BS, SMC, USA) such that \ac{RoSEv2.0} has 12 independent inputs to control the actuation of \ac{RoSEv2.0}. In contrast with \ac{RoSE}, \ac{RoSEv2.0} is developed with embedded sensing capability. The significant variation between \ac{RoSE} and \ac{RoSEv2.0} is the outer casing construction, which is replaced with a transparent \ac{PDMS} (SYLGARD 182, Dow, USA) casing in \ac{RoSEv2.0} (\autoref{fig_rose_schematic}). An array of \ac{TOF} (VL6180X, STMicroelectronics, Switzerland) sensors are placed on top of the \ac{PDMS} layer to measure the conduit deformation laterally from the outside. 

\begin{figure}[!bt]
	\myfloatalign
	{\includegraphics[width=1\linewidth]{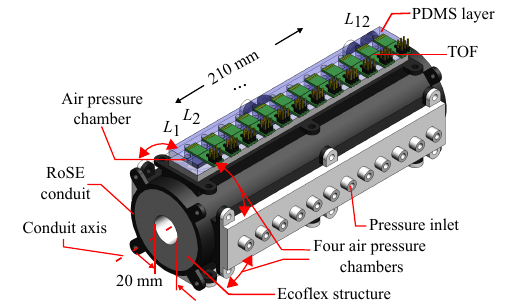}}
	\caption[Isometric view of \ac{RoSEv2.0} \ac{CAD} model.]{Isometric view of \ac{RoSEv2.0} \ac{CAD} model. 
	} 
	\label{fig_rose_schematic}
\end{figure}

The \ac{DTSINDYC} models $ M_1 $ and $ M_2 $ based on \ac{TOF} and \ac{VPS}, respectively (\autoref{tab2_model_characteristics}) for implementing the \ac{MPC} are built on a small number of datasets collected from the \ac{RoSEv2.0} sensors and aims to discover the underlying \acp{DTDE} given by \eqref{eq_dtsindyc} such that the equations best-fit the captured dataset. 
\begin{table}[!tb]
	\renewcommand{\arraystretch}{1.3}
	\caption[Characteristics of \ac{DTSINDYC} models.]{Characteristics of the \ac{DTSINDYC} models.}  \label{tab2_model_characteristics}
	\centering
	\label{tab2_model_characteristics}
	\resizebox{\columnwidth}{!}{
		\begin{tabular}{l c c c c}
			\hline\hline \\[-3mm]
			Model & \multicolumn{1}{c}{\pbox{2cm}Sensor used} & \multicolumn{1}{c}{\pbox{2cm}State variables} & \multicolumn{1}{c}{\pbox{2cm}{Representation}} & \multicolumn{1}{c}{\pbox{2cm}{Control actuation\\ \hphantom{spaces} type}}  \\ [1.6ex]\midrule
					
					$ M_1 $  & TOF  & $ x_{k,1} $, $ x_{k,2} $, $ x_{k,3} $ & Displacement (mm)&Peristaltic \\
					$ M_2 $  & VPS  & $ x_{k,1} $, $ x_{k,2} $, $ x_{k,3} $ & Pressure (kPa)&Peristaltic \\
\hline\hline
		\end{tabular}
	}
\end{table} 
\begin{align}
	\label{eq_dtsindyc}
	{\mathbf{x}}_{k+1}={\mathbf g}({\mathbf{ x}}_{k},{\mathbf{ u}}_k)	
\end{align}

In \eqref{eq_dtsindyc}, $ {\mathbf{ u}}_k\in \mathbb{R}^{n}$, and $  {\mathbf{ x}}_k\in\mathbb{R}^{n}$ represent the digital values for pressure command applied to the pneumatic valves at $ k^{th} $ time step to govern the \ac{RoSEv2.0} actuation and its corresponding \ac{TOF} or \ac{VPS} measurements, respectively. $ \mathbf{g}(.) $ represents a nonlinear polynomial function consisting of first and second order polynomial, and constant terms, mapping $ \mathbb{R}^{n}\times\mathbb{R}^{n}\rightarrow\mathbb{R}^{n} $.
Eq. \ref{eq_dtsindyc} also represents the predictive nature of the \ac{DTSINDYC} models of \ac{RoSEv2.0}, which will be used for implementing the \ac{MPC}. At any time-instant $ k $, given $ {\mathbf{ x}}_k $ and $ {\mathbf{ u}}_k $ of \ac{RoSEv2.0}, the future of the
\ac{RoSEv2.0} states denoted by $ {\mathbf{ x}}_{k+1} $ can be predicted. In this research, $ \mathbf{g}(.) $ represents two different \ac{DTSINDYC} models $ M_1 $ and $ M_2 $, generated from the \ac{VPS} and \ac{TOF} captured data, respectively.

\autoref{tab2_model_characteristics} provides the description and physical representation of the state variables associated with the models. The models are used in the \ac{MPC} for controlling conduit displacement or chamber air pressure in case of peristaltic actuation \cite{bhattacharya2020rose}. Since at least three RoSEv2.0 layers are required to match the human esophagus peristalsis wavelengths \cite{Dirven2015}, thus by emphasizing the TOF sensors measurement accuracy, layer $ L_5 $, $ L_6 $, and $ L_7 $ were chosen. The \ac{MPC} of \ac{RoSEv2.0} using $ M_1 $, or $ M_2 $ is implemented to compute the control law $ \mathbf{ u}(.|{\mathbf{ x}}_k)= \{{\mathbf{ u}}_{j+1},\cdots,{\mathbf{ u}}_{j+k},$ $ \cdots, {\mathbf{ u}}_{j+N_u}\} $ at any time-step $ j $, for achieving a prescribed peristaltic actuation profile of the robot's conduit over a control horizon $ N_u $ and prediction horizon $ N_p $, provided the current \ac{TOF} or \ac{VPS} measurement $ {\mathbf{ x}}_j $ is available. Hence, the implicit feedback law for controlling the valves and hence, controlling the \ac{RoSEv2.0} can be written as
\begin{align}
	\label{eq_mpc_control_law}
	\mathbf{ u}(j+1|\mathbf{ x}_j)=\mathbf{ u}_{j+1}	
\end{align}  

where $ {\bf u}_{j+1} $ is the first row in the optimized actuation sequence, representing pressure commands applied to the  $ L_5 $, $ L_6 $ and $ L_7 $ pneumatic valves, beginning at the initial condition $ {\bf x}_{j} $.

\section{RoSEv2.0 Fabrication, Actuation, Sensing, and State Definition}\label{sec_rosev2_construction}

\ac{RoSEv2.0} is built utilizing custom-designed mold and housings.
\ac{RTV} silicone rubber material (Ecoflex 0030, Smooth-on, USA) in its uncured state is poured into the housing, rod, and chamber-placeholder assembly and left seated for vulcanization at room temperature (\autoref{fig_rosev2_constructio} (a)). After hardening of the silicone rubber, the chamber placeholders are removed (\autoref{fig_rosev2_constructio} (b)). Simultaneously, the \ac{PDMS} (SYLGARD 182, Dow, USA) outer layer is also cast and removed from its molds in a similar manner (\autoref{fig_rosev2_constructio} (c) and (d)). Subsequently,  side  $ S_{1} $ of \ac{RoSEv2.0} and the \ac{PDMS}  layer are bonded together with an adhesive (Sil-Poxy, Smooth-On, USA) (\autoref{fig_rosev2_constructio} (e)). 
With the same adhesive on  $ S_{1} $, an array of \ac{TOF} sensors is attached on top of the \ac{PDMS} layer (\autoref{fig_rosev2_constructio} (f)). Outer covers, pneumatic tubes, and tees are attached to \ac{RoSEv2.0}  to assure that the silicone rubber deformation occurs inwards (\autoref{fig_rosev2_constructio} (g)).

\begin{figure*}[!t]
	\myfloatalign
	{\includegraphics[width=1\textwidth]{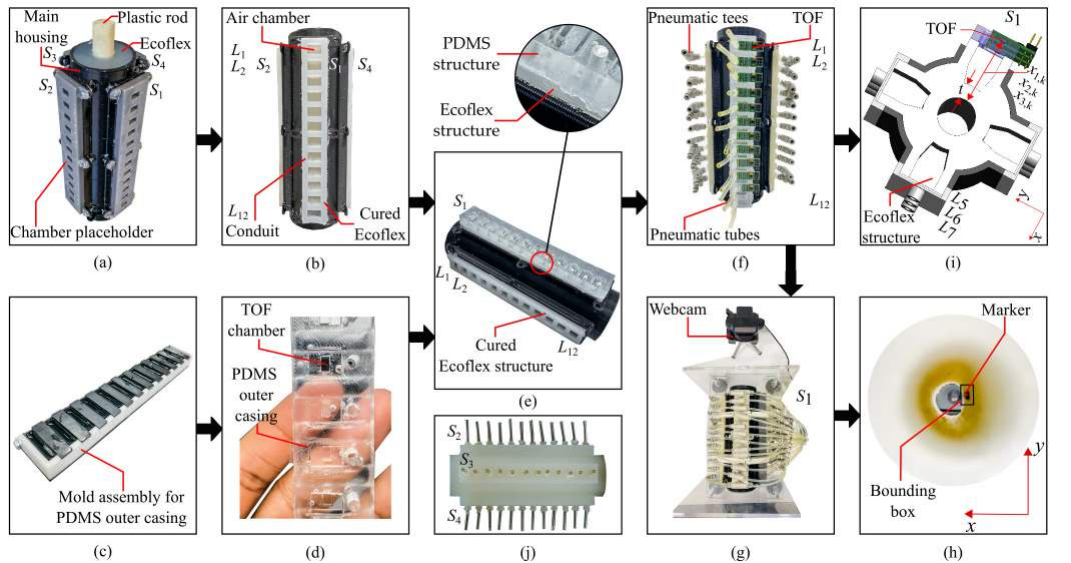}}
	\caption[Construction of \ac{RoSEv2.0}.]{(a) to (f) Construction of \ac{RoSEv2.0}. (g) to (h) Webcam setup for \ac{TOF} calibration. (i) \ac{RoSEv2.0} sectional view for defining the states. (j) {Complete rigid-boundary less Ecoflex structure of the RoSEv2.0 conduit.}} 
	\label{fig_rosev2_constructio}
\end{figure*}

The entire \ac{RoSEv2.0} conduit $ L_{1} $ to $ L_{12} $ does not contain any rigid skeletal boundary, and when pressurized with air, the whorl of four chambers per layer inflate and occlude, mimicking circular muscle activation. The complete rigid-boundary less Ecoflex structure of the conduit is shown in \autoref{fig_rosev2_constructio} (j).
{The pneumatic valves corresponding to each RoSEv2.0 layer are interfaced with a Raspberry Pi 4 Model B through an }\ac{ADC} and \ac{DAC} board. {The pneumatic valves have a set pressure range of $1 -  500 $ kPa, operates with an input signal range of $ 0 - 10  $ V, characterizing a resolution of $ 50 $ kPa.V$ ^{-1} $. Since the $ 8 $-bit DACs connected to the valves have a resolution of $ 0.02 $ V.step$ ^{-1} $ thus, each digital value applied to the DAC-valve assembly will generate a pressure with $ 1 $ kPa.step$ ^{-1} $.  For complete occlusion of the RoSEv2.0 conduit, maximum pressure of $ 47  $ kPa is required, which is regarded as the upper limit for the control signals generated by the MPC. }

Custom firmware protocols for generating peristalsis, written in Python 3.7, are developed on the Pi to actuate the robot in open-loop, with continuous, independent variable pressure trajectories as discussed in \ref{sec_problem_formulation}. By utilizing the peristaltic protocols, training and validation data for the \ac{DTSINDYC} models are collected to derive and test $  M_1$ and $ M_2 $.

The \acp{TOF} are initially calibrated for range-offset and crosstalk to operate through the \ac{PDMS}. The \acp{TOF} are further calibrated with a webcam (C922, Logitech, Switzerland) by implementing Python OpenCV CSRT tracker algorithm (\autoref{fig_rosev2_constructio} (g) and (h)). For webcam calibration, two hemispherical $ 4 $ mm retroreflective markers are placed, one for tracking and another for reference. 
Both \ac{TOF} and webcam measurements are taken simultaneously by inflating the \ac{RoSEv2.0} chambers with the peristaltic actuation protocol. The time-series measurements are compared to determine additional offsets and scaling factors.

The measured \ac{RoSEv2.0} conduit displacement data at the center of $ L_5 $, $ L_6 $, and $ L_7 $ air pressure chamber of side $ S_1 $ are considered as the states of \ac{RoSEv2.0}. 
Hence, three discrete states $x_{k,1}$, $x_{k,2} $, and $x_{k,3}  $ corresponding to layer $ L_5 $, $ L_6 $, and $ L_7 $ are defined as \ac{RoSEv2.0} states (\autoref{fig_rosev2_constructio} (i)). Each state represents $ x $-axis conduit displacement at $ L_5 $, $ L_6 $, and $ L_7 $ of $ S_1 $, respectively.
For modeling simplification, it is assumed that the conduit thickness $ t $ remains unchanged throughout its deformation. Additionally, it is also considered that the deformation of the adjacent chambers in each layer is symmetric, and displacement along $ y $ and $ z $-axis are negligible \cite{bhattacharya2019sparse}. 


\section{Discrete time SINDYC}\label{sec_dtsindyc}


The original \ac{SINDYC} algorithm can be extended to discrete-time dynamical systems, represented by \eqref{eq_dtsindyc} \cite{brunton2016extracting}. There are two primary reasons for implementing the \ac{DTSINDYC} over the conventional one. Firstly, the calculation of a derivative from noisy data is not required for the \ac{DTSINDYC} algorithm, which is a significant advantage over its continuous-time counterpart.  Secondly, a discretized model of \ac{RoSEv2.0} is necessary for the implementation of \ac{MPC}. 

The data collected from the \acp{VPS} or the \ac{TOF} ($x_{k,1}$, $x_{k,2} $, and $x_{k,3}  $, (\autoref{fig_rosev2_constructio} (i)) sensors are arranged into two $ m\times n $ matrices $ \mathbf{X}^{m}_{1}=[ \begin{matrix}
	{{\mathbf{x}}_{1}^{T}},{{\mathbf{x}}_{2}^{T}},\cdots,{{\mathbf{x}}_{m}^{T}}
\end{matrix} ]^{T}$ and  
$\mathbf{X}^{m+1}_{2}=[ \begin{matrix}
	{{\mathbf{x}}_{2}^{T}},{{\mathbf{x}}_{3}^{T}},\cdots,{{\mathbf{x}}_{m+1}^{T}}
\end{matrix} ]^{T} $. Likewise, the input data are stored in an $ m\times l $ matrix   $ 	\mathbf{U}=[ \begin{matrix}
{{\mathbf{u}}_{1}^{T}},{{\mathbf{u}}_{1}^{T}},\cdots, {{\mathbf{u}}_{m}^{T}}
\end{matrix}]^{T}$. Variables $ m $, $ n $, and $ l $ represents the number of samples, state variables ($x_{k,1}$, $x_{k,2} $, and $x_{k,3}  $, $ n=3 $), and inputs, respectively and ${{\mathbf{x}}_{q}^{}}=[x_{q,1},x_{q,2},\cdots,x_{q,n}]^{T}|q\in\{1,\cdots,m+1\}$.



In this research, the number of inputs is considered equal to the number of outputs ($ l=n $), where inputs and outputs are digital pressure commands to the valves and \ac{TOF} or \ac{VPS} measurements, respectively. The \ac{DTSINDYC} algorithm implements a library of candidate functions matrix $ \boldsymbol\Theta $ of order $ m\times p $, consisting of constant and polynomial (of different order) terms. The library matrix $ {\boldsymbol\Theta} $ can be defined as

\begin{align}
	\label{eq_lib_matrix}
	{\boldsymbol\Theta}({\mathbf X_1^{m}},{\mathbf U})={{\left[
				1,{\mathbf X_1^{m}},\mathbf{U},{\mathbf X_1^{m}}\otimes{\mathbf X_1^{m}},{\mathbf X_1^{m}}\otimes\mathbf{U}, \cdots
			\right]}} 	
\end{align}

where, $ {\mathbf{X}}_{1}^{m} \otimes \mathbf{U} $  defines all the possible product combinations of the components of $ \mathbf{X} $ and $ \mathbf{U} $.  Each column of $ {\boldsymbol\Theta}({\mathbf X_1^{m}},{\mathbf U}) $ represents the potential candidate for the final right-hand side expression for \eqref{eq_dtsindyc}. There is an enormous opportunity for decision in building the library function of nonlinearities. Thus, the system in \eqref{eq_dtsindyc} can be written as

\begin{align}
	\label{eq_Theta_Upsilon}	
	\mathbf{X}_2^{m+1}={\boldsymbol\Theta}({\mathbf X_1^{m}},{\mathbf U})\mathbf\Upsilon
\end{align}

where $ \mathbf{\Upsilon}=\left[\boldsymbol{\upsilon}_0,\boldsymbol{\upsilon}_1,\cdots,\boldsymbol{\upsilon}_{n-1}\right] $ is a sparse matrix of coefficients of order $ p\times n$ such that the sparse vectors $\{{\boldsymbol \upsilon}_i\}_{i=0}^{n-1}  $, lie in the subspace of $ {\mathbb R}^{p} $. The sparsity of $ \boldsymbol\upsilon_i $ selects the active terms from $ {\boldsymbol\Theta}({\mathbf X_1^{m}},{\mathbf U}) $ to identify the $ {{g}}_i $ of $ {x_{k+1,i}}={{g}}_i({\mathbf{x}}_k,{\mathbf{u}}_k) $ in \eqref{eq_dtsindyc}. 
The \ac{DTSINDYC} model in reveals the underlying dynamics of \ac{RoSEv2.0} in the form of \acp{DTDE}. Since only a few of the candidates will appear in the row of $ \mathbf{g} $ in \eqref{eq_dtsindyc}, a sparse regression problem can be defined and $ \mathbf\Upsilon $ can be determined by optimizing it \cite{brunton2016discovering}. 
After the determining $ \mathbf\Upsilon $, \ac{DTSINDYC} model of \ac{RoSEv2.0} can be written as

\begin{align}
	\label{eq_all_st_eq}	
	\mathbf{x}_{k+1}={\mathbf g}({\mathbf{ x}}_{k},{\mathbf{ u}}_k)={\boldsymbol\Upsilon}^{T}{\boldsymbol\Theta}({{\mathbf x}_k^{T}},{{\mathbf u}_k^{T}})^{T}
\end{align}

\section{Overview of Model Predictive Control in RoSEv2.0}\label{sec_mpc_theory}
The \ac{MPC} optimization is perfomed over a receding prediction horizon by forward shifting the prediction horizon at every time step $ N_p $. At $ j^{th} $ time instant, the \ac{MPC} applies a reference signal $ {\mathbf{ x}}_{j}^{(ref)}$ to the \ac{CFM} block (\autoref{fig_mpc}). The \ac{CFM} block minimizes a cost function $ J $ over $ N_p $. The output of the \ac{CFM} is either used as an input to the valves or  \ac{DTSINDYC} model ($ M_1$ or $ M_2 $, \autoref{tab2_model_characteristics}). Since the model in \ac{MPC} is used for prediction; thus, the model output at any instant $ j $ is denoted by $ {\hat {\mathbf{ x}}}_{j}$. Between $ j $ and $ {j+1}^{} $ time step, the \ac{SPDT} switch $ Sw_1 $ is set to the \ac{DTSINDYC} model whose output is used to calculate a set of optimal control values $\mathbf{ u}_{.|\mathbf{ x}_j}:= \{\mathbf{ u}_{j+1},\cdots,\mathbf{ u}_{j+k},\cdots,\mathbf{ u}_{j+N_c} \}$ over a control horizon $ N_c $. The behavior of \ac{RoSEv2.0} over $ N_p $ is estimated by its model. To determine the sequence of estimations $ \{\mathbf{ \hat x}_{j},\mathbf{ \hat x}_{j+1},\cdots,\mathbf{ \hat x}_{j+k},\cdots,\mathbf{ \hat x}_{j+N_p-1} \}$, $ Sw_2 $ is initially set to the \ac{RoSEv2.0} \ac{TOF} or \ac{VPS} for $\mathbf{ \hat x}_{j}$ and then to the model output for the subsequent predictions. Next, when the \ac{CFM} algorithm has solved for the best input, $ S_1 $ is moved to the valves and first control value from $ \mathbf{ u}_{j+1} $ is applied to the valves. At next time instant $ j+1$, the computation is repeated with horizon moved by one time-step. The implicit control law for each time-step is given by \eqref{eq_mpc_control_law}. The cost function optimized by the \ac{CFM} block at each time-step can be written as

\begin{figure}[!t]
	\myfloatalign
	{\includegraphics[width=0.90\linewidth]{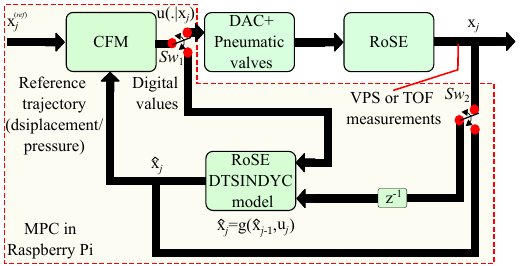}}
	\caption[ of \ac{DTSINDYC}-\ac{MPC} explaining its functional description.]{Block diagram of \ac{DTSINDYC}-\ac{MPC} explaining its functional description. 
	} 
	\label{fig_mpc}
\end{figure}

\begin{align}
	\label{eq_mpc_control_optimization}
	\begin{split}
		\arg \underset{{\mathbf{ u}_{.|\mathbf{ x}_j} }}{\mathop{\min }}\,&{{J}_{}}({\mathbf{ x}_j })=	\arg \underset{{\mathbf{ u}_{.|\mathbf{ x}_j} }}{\mathop{\min }} \left\{{\|{\mathbf {\hat x}}_{j+N_p} - {\mathbf {x}}_{j+N_p}^{(ref)}\|_{\mathbf{\tilde Q}}}+\right.\\& {\sum\limits_{k=0}^{N_p-1}}{\|\mathbf{ \hat x}_{j+k}-\mathbf{ x}^{(ref)}_{{j+k}}\|^{2}_{\mathbf{Q}}}\\
		&\left.{\sum\limits_{k=0}^{N_c-1}}\left({\|\mathbf{ u}_{j+k}\|^{2}_{\mathbf{R_u}}}+{\|\mathbf{ u}_{j+k}-\mathbf{ u}_{j+k-1}\|^2_{\mathbf{R_{\Delta u}}}}\right)\right\}\\
	\end{split}
\end{align}

\begin{align}
	\label{eq_mpc_control_law_constraints_a}
	\begin{split}
		\text{s.t. } {\mathbf{ \hat x}_{j+k}}=\mathbf{ g}({\mathbf{ \hat x}_{j+k-1}},\mathbf{ u}_{j+k-1})\text{, }
		\mathbf{ u}_{lb}\le \mathbf{ u}_k \le \mathbf{ u}_{ub}  
	\end{split}
\end{align}
At every iteration (time-step) $ j $, CFM minimizes the functional $ J $ in \eqref{eq_mpc_control_optimization}, which includes a terminal cost at $ {\mathbf {\hat x}}_{N_p} $. The leftmost term in \eqref{eq_mpc_control_optimization} penalizes the root mean square deviations of the \ac{RoSEv2.0} model predicted trajectory $ \mathbf{ \hat x}_{{j+k}} $ (given by \eqref{eq_mpc_control_law_constraints_a}), from the reference trajectory $ \mathbf{ x}^{(ref)}_{{j+k}} $. Control output sequence $ \mathbf{  u}_{{j+k}} $ and change in control output sequence $ \left(\mathbf{  u}_{{j+k}}- \mathbf{  u}_{{j+k-1}}\right)$ are also penalized by the middle and the rightmost terms in \eqref{eq_mpc_control_optimization}, respectively. The range of $ {\mathbf{ u}}_k $ at every iteration of \ac{MPC} is constrained by control bounds  ($ \mathbf{ u}_{lb} $ and $ \mathbf{ u}_{ub} $ in \eqref{eq_mpc_control_law_constraints_a}). In \eqref{eq_mpc_control_optimization}, the weight matrix $ {\mathbf{Q}} $ multiplies the deviation of the controlled variable prediction at $ j+k  ^{th} $ instant from the respective reference point. ${\mathbf{R_u}} $ and $ {\mathbf{R_{\Delta u}}} $ penalize the value of the control vector and future control moves, respectively. In \eqref{eq_mpc_control_optimization}, ${\mathbf{Q}} $ and ${\mathbf{{\tilde Q}}} $ are positive semi-definite, $ {\mathbf{R_{\Delta u}}} $ and ${\mathbf{R_u}} $ are positive definite matrices.

Closed-loop stability for online MPC is an issue if receding horizon predictive control is implemented in such a way that requires on-line redesign. In MPC, the terminal cost and terminal constraint ensure recursive feasibilty and stability of the closed-loop system, but the terminal constraint reduces the region of attraction. In practice, a sufficient prediction horizon size provides necessary condition for the stability and it makes the system stable and feasible in a (large) neighborhood of the origin. On the basis of horizon length, MPC can be categorized into finite and infinite horizon. 


Unlike finite horizon MPC, closed loop stability is inherent in infinite horizon constraint less MPC, since at  $ {k+1}^{th} $ instant, 
no new information enters the optimization problem, so the optimal trajectory determined at $ k+1 $ is same as that of the tail of $ k ^{th}$ instant \cite{Maciejowski_Predictive}.

\begin{theorem*}
	Suppose that MPC is obtained for a plant, defined by the nonlinear time-invariant equation in \eqref{eq_dtsindyc}, by minimizing the cost function in \eqref{eq_mpc_control_optimization}. If the plant model remains unchanged and stable throughout the computation of the control law, and settling to the reference point occurs within the prediction horizon, then despite the use of finite horizon, MPC is closed loop stable.
\end{theorem*}

\begin{proof}
	The infinite horizon formulation of \eqref{eq_mpc_control_optimization} can be written as
	
	\begin{gather}
		\label{eq_mpc_control_optimization_2}
		\begin{split}
			&V_{mpc}(j)=	{\sum\limits_{k=N_c}^{\infty}}{\|\mathbf{ \hat x}_{j+k}-\mathbf{ x}^{(ref)}_{{j+k}}\|^{2}_{{\mathbf{Q}}}}+{\sum\limits_{k=0}^{N_c-1}}{\|{\mathbf{ \hat x}}_{j+k}^{2}}-\\
			&{\mathbf{x}}_{j+k}^{(ref)}\|_{\mathbf Q}+{\sum\limits_{k=0}^{N_c-1}}\left({\|\mathbf{ u}_{j+k}\|^{2}_{\mathbf{R_u}}}+{\|\mathbf{ u}_{j+k}-\mathbf{ u}_{j+k-1}\|^2_{\mathbf{R_{\Delta u}}}}\right)
		\end{split}
	\end{gather}
	
	Let, $ \mathbf{ u}_{j+k} = {\mathbf u}_c$ such that $ {\mathbf u}_c $ is the control input required to fix $ \mathbf{ \hat x}_{j+k} $ at $ \mathbf{ \ x}^{(ref)}_{j+k} $ $\forall k\ge N_c-1$. Hence, change in $  \mathbf{ u}_{j+k} $ is also zero $\forall k\ge N_c-1$. Assume ${\mathbf x}^{(ref)}=\mathbf 0$ and hence, $ {\mathbf u}_c = \mathbf 0$ is required to drive the model output to zero. Therefore, for $ k>{N_c} $ only first term of \eqref{eq_mpc_control_optimization_2} given by $\sum_{k=N_c}^{\infty}{\|{{\mathbf {\hat x}}}_{j+k}\|}^{2}_\mathbf{Q}$ remains, which is a matrix geometric mean that converges if the plant is stable. In the vicinity of the origin and zero input, a nonlinear system like \eqref{eq_dtsindyc} can be approximated as that of a linear system $ {\mathbf x}_{k+1}\approx\mathbf {A}{\mathbf x}_{k} $. If the linearized system is stable then all the eigenvalues of $ \mathbf A $ lies inside the unit circle and hence, $\sum_{k=N_c}^{\infty}{\|{{\mathbf {\hat x}}}_{j+k}\|}^{2}_\mathbf{Q}$ can be written as
	
	\begin{align}
		\label{eq_lyapunov equation_a}
		\begin{split}
			\sum_{k=N_c}^{\infty}{\|{{\mathbf {\hat x}}}_{j+k}\|}^{2}_\mathbf{Q}={{\mathbf {\hat x}}}_{j+N_c}^{T}\left[\sum_{i=0}^{\infty}{({\mathbf A}^{T}})^{i}{\mathbf Q}
			{\mathbf A}^{i}\right]{{\mathbf {\hat x}}}_{j+N_c}
		\end{split}
	\end{align}	
	
	In \eqref{eq_lyapunov equation_a}, let $  {\mathbf {\tilde Q}} = \sum_{i=0}^{\infty}{({\mathbf A}^{T}})^{i}{\mathbf Q}
	{\mathbf A}^{i}$ then, $ {\mathbf {\tilde Q}} $ satisfies the discrete-time Lyapunov equation \cite{sahi2004} given by 
		
	\begin{align}
	\label{eq_lyapunov equation}
	\begin{split}
		{\mathbf A}^{T}{\mathbf {\tilde Q}}{\mathbf A}={\mathbf {\tilde Q}}-{\mathbf Q} 
	\end{split}
	\end{align}
	In \eqref{eq_lyapunov equation}. $ {\mathbf {\tilde Q}} $ is a positive semi-definite matrix, if $ {\mathbf {Q}} $ is positive semi-definite. Therefore, \eqref{eq_mpc_control_optimization_2} can be re-written in a generalized form as
	\begin{eqnarray}
		\label{eq_mpc_control_optimization_3}
		\begin{split}
			&V_{mpc}(j)={\|\mathbf{ \hat x}_{j+{N_c}}\|^{2}_{{\mathbf{\tilde Q}}}}+	{\sum\limits_{k=0}^{N_c-1}}{\|{\mathbf{ \hat x}}_{j+k}^{2}}\|_{\mathbf Q}\\&+{\sum\limits_{k=0}^{N_c-1}}\left({\|\mathbf{ u}_{j+k}\|^{2}_{\mathbf{R_u}}}+{\|\mathbf{ u}_{j+k}-\mathbf{ u}_{j+k-1}\|^2_{\mathbf{R_{\Delta u}}}}\right)
		\end{split}
	\end{eqnarray}
	When ${\mathbf x}^{(ref)}=\mathbf 0$, \eqref{eq_mpc_control_optimization_3} and \eqref{eq_mpc_control_optimization} are similar, where the predictive control problem in \eqref{eq_mpc_control_optimization_3} is formulated over $ N_c $. 
\end{proof}




\section{MPC Design and Implementation in RoSEv2.0} \label{sec_mpc_design}
This section ties together \ref{sec_rosev2_construction} to \ref{sec_mpc_theory} in order to design and implement the \ac{MPC} in \ac{RoSEv2.0}.
In the domain of mathematics, the biological peristaltic waves are typically modeled as per the sinusoidal waveform given by \eqref{eq_wavefrontmodel} \cite{misra2001mathematical,takagi2011peristaltic}. 

\setlength\arraycolsep{1.0pt}
\begin{eqnarray}
	\label{eq_wavefrontmodel}
	H(z,t)=\begin{cases}
		\epsilon+ \frac{a}{2}\left[1-cos\left(2\pi\frac{z-ct}{w}\right)\right], {ct\le z\le (ct+w)} \\
				\epsilon\hspace{104pt}, \text{elsewhere }
	\end{cases}
\end{eqnarray}


In \eqref{eq_wavefrontmodel}, $\epsilon$ is minimum radius (mm), $z$ is axial displacement (mm), $t$ is time ($s$), $H(z,t)$          		is conduit's radius as a function of $ x $ and $t$, $a$ is amplitude of the wave (mm), $c$ is velocity of the travelling wave (mm.s$ ^{-1}$), and $w$ wavelength (mm). \autoref{eq_wavefrontmodel} is a modified version of the orginal equation described by \cite{misra2001mathematical,Dirven2015} to incorporate the entire ascending half-cycle (full wavelength). By considering $ z=z_d $, ($ z_d $ is the axial adjacent chamber spacing, which is $ 15 $ mm), $ t_d=z_d/c $ continuous time frequency $ f_d=c/w $, and $ t=kT_s $, discrete-time form of \eqref{eq_wavefrontmodel} can be written as

\begin{align}
\label{eq_wavefrontmodel_4}
	H(k)=\epsilon+\frac{a}{2}\left[1-cos\left\{\omega\left({k-\frac{t_d}{T_s}}\right)\right\}\right]
\end{align}

where, discrete frequency $\omega=2\pi f_d/F_s  $, sampling frequency $ F_s = 1/T_s$, and $ t_d/T_s\le k \le (t_d/T_s + 1/T_sf_d) $. The peristaltic waveform profile, given by  \eqref{eq_wavefrontmodel_4}
can be accomplished by controlling time-varying pressure trajectories into the adjacent whorl of chambers simultaneously (\autoref{fig_mpc_implementation}). The pneumatic valves associated with each layer of \ac{RoSEv2.0}, command the pressure profile in each layer. The role of the \ac{MPC} is to find out the optimal control law in terms of digital values, which are supplied to the valves to achieve peristaltic wave shapes, defined by the parameters ($ c $ and $ w/2 $) of \eqref{eq_wavefrontmodel_4}. 

Next, three time-shifted versions of \eqref{eq_wavefrontmodel_4} for layer $ L_5 $, $ L_6 $, and $ L_7 $ of \ac{RoSEv2.0} are formulated ($ x^{(ref)}_{k,n}=H_n(k)|n=1,2,3$) as per \eqref{eq_wavefrontmodel_5}. 

\begin{align}
	\label{eq_wavefrontmodel_5}
	x^{(ref)}_{k,n}=\epsilon+\frac{a}{2}\left[1-cos\left\{\omega\left({k-(n-1)\frac{t_d}{T_s}}\right)\right\}\right]
\end{align}
In \eqref{eq_wavefrontmodel_5}, $ (n-1){t_d}/{T_s}\le k \le (n-1)\left({t_d}/{T_s}+{1}/{Tsf_d}\right)|n=1,2,3 $. At a time-instant $ j $, the \ac{MPC} loop receives three present  ($\{x^{(ref)}_{j,n}=H_{j,n}|n=1,2,3\}$, \eqref{eq_wavefrontmodel_5}) and a set of future reference signal values, and feedback value ($\{x^{}_{j-1,n}|n=1,2,3\}$) from \ac{TOF} sensors or \acp{VPS} located at layer $ L_5 $, $ L_6 $ and $ L_7 $ of \ac{RoSEv2.0} (\autoref{fig_mpc_implementation}). The blocks within the red dashed section of \autoref{fig_mpc} represents the expanded \ac{MPC} loop in \autoref{fig_mpc_implementation}. By following the minimization procedure discussed in \ref{sec_mpc_theory}, the \ac{MPC} loop generates three control values $ u_{j+1,1} $, $ u_{j+1,2} $ and $ u_{j+3,3} $ , which are stored in the $ L_2 $, $ L_3 $ and $ L_4 $ index of the accumulator. $  L_1 $ index is not used since bolus feeder pipes are inserted in \ac{RoSEv2.0} conduit and the pipes coincide with $  L_1 $ and $  L_{12} $ conduit locations. Hence, to keep the orientation of the feeder pipes unchanged, layers $  L_1 $ and $  L_{12} $ are not actuated.

\begin{figure}[!bt]
	\myfloatalign
	{\includegraphics[width=0.95\linewidth]{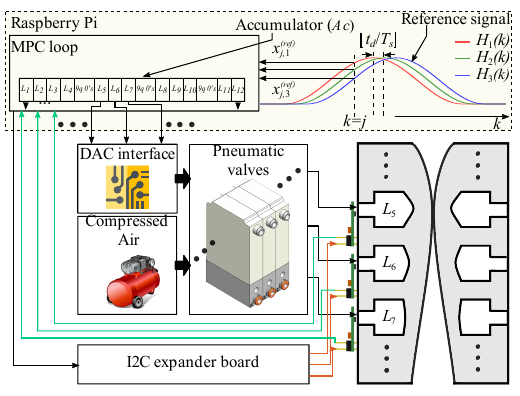}}
	\caption[Block diagram of \ac{DTSINDYC}-\ac{MPC} design and implementation in \ac{RoSEv2.0}.]{Block diagram of \ac{DTSINDYC}-\ac{MPC} design and implementation in \ac{RoSEv2.0}.} 
	\label{fig_mpc_implementation}
\end{figure}

In the next iteration ($j+1$), a new set of control signals ($ u_{j+2,1} $, $ u_{j+2,2} $ and $ u_{j+2,3} $) are generated and the earlier set of signals are right-shifted by three in the accumulator. It takes $3q+1$ shifts (iterations) for $ u_{j+1,1} $, $ u_{j+1,2} $ and $ u_{j+1,3} $ to reach $ L_5 $, $ L_6 $ and $ L_7 $ respectively, where $q$ is the number of time steps between the peaks of consecutive reference signal given by discrete time delay $q=\lfloor{t_d/T_s}\rfloor$ of \eqref{eq_wavefrontmodel_4}. If $ c=20 $ mm/s, $ x_d =15$ mm $ w=120 $ mm and $ T_s =0.1$ s, then $ t_d =0.75$ s, $ f_d=1/6 $ s$ ^{-1} $ and $ q= 7$. Here, $ Ts $ is the time taken by an \ac{MPC} loop  to complete one cycle, which can vary upon conditions like reference signal type and model complexity, prediction horizon and input bounds. With this kind of approach, the \ac{MPC} has been made robust to adapt for various peristalsis waveform speeds and wavelengths. The only downside to this approach is difficulty in achieving higher wave speeds if $ T_s $ increases. For best and fast optimization, the \ac{MPC} was tuned for various prediction horizon lengths, weight matrices and control input bounds.

\section{Results}

\subsection{DTSINDYC Models for Peristaltic Actuation}\label{section_dtsindyc_models}

This section presents the results with the \ac{DTSINDYC} modeling methodology presented in \ref{sec_dtsindyc}. Layer $ L_5 $, $ L_6 $, and $ L_7 $ of the conduit were actuated with time-shifted staircase waveforms of varying amplitude and time (\autoref{fig_dtsindyc_tof_states_3} (a)) to apply \ac{DTSINDYC}. It was found that to profile the conduit as per \eqref{eq_wavefrontmodel_5} with different waveform speed and diameter; the conduit layers must be actuated with the waveforms as mentioned earlier. The amount of time-shift governs the speed of the peristalsis in \ac{RoSEv2.0} and the range of the peristalsis speed was obtained from the \ac{RoSE} articulography results presented by Dirven et al. \cite{Dirven2015}. {In this research, each step of digital value applied to the DACs and valves assembly has generated a corresponding pressure with $ 1 $ kPa.step$ ^{-1} $} (Refer to \ref{sec_rosev2_construction}). The maximum actuation pressure has been restricted to $ 47 $ kPa, which is the pressure needed for full \ac{RoSEv2.0} conduit occlusion.

\begin{figure}[!t]
	\myfloatalign
	{\includegraphics[width=0.8\linewidth]{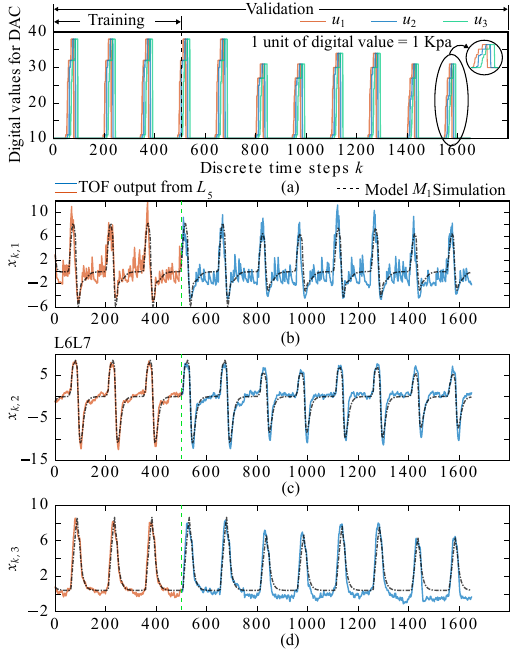}}
	\caption[Plots illustrating the training and validation of models $M_1$ .]{Plots illustrating the training and validation of model $M_1$. (a) The input dataset applied to the $L_5$, $L_6$ and $L_7$ \acp{DAC} to generate the output dataset shown by blue and orange plots in (b) to (d). (b) to (d) The dashed black line represents the one-step ahead prediction of $M_1$.} 
	\label{fig_dtsindyc_tof_states_3}
\end{figure}

A dataset of size $4350 \times 3$ comprising various amplitude and time-shifted staircase waveform levels (\autoref{fig_dtsindyc_tof_states_3} (a)) and their respective \ac{TOF} output was collected (orange and blue waveforms in \autoref{fig_dtsindyc_tof_states_3} (b) to (d)). \hlcyan{The blue and orange color plots in} \autoref{fig_dtsindyc_tof_states_3} \hlcyan{(b) to (d) also depict the open-loop response of RoSEv2.0, which was collected as a time-series dataset and utilized for training the DTSINDYC model $ M_1 $.}

\hlcyan{In the first five cycles of} \autoref{fig_dtsindyc_tof_states_3} \hlcyan{(b) to (d), layer $ L_5 $, $ L_6 $, and $ L_7  $ have shown displacements up to $ 10\pm1.6 $, $ 8.2\pm0.1 $, and $ 8.2\pm0.1 $ mm, respectively, when actuated with digital DAC values of $ 38 $ (or $ 38 $ kPa). In the absence of a closed-loop controller, it was challenging to determine the needed digital DAC value for full conduit closure (i.e., displacement of $ 10$ mm) without damaging the robot. The MPC for RoSEv2.0 is implemented to determine the optimal digital DAC values and hence pressure values required in the air chambers of each layer to achieve the prescribed peristalsis. }


Model $ M_1 $ discovered by applying \ac{DTSINDYC} \cite{desilva2020pysindy} to the collected dataset is given as

\begin{eqnarray}
	\label{eq_M3}
	\begin{bmatrix}
		x_{k+1,1} 
		\\x_{k+1,2}
		\\x_{k+1,3}  
	\end{bmatrix}
	&=\begin{bmatrix}
		0.85 & 0.04 & (-0.01-0.02x_{k,2})\\ 
		0.02 & 0.84 & (-0.17-0.02x_{k,2})\\ 
		0 &0  & 0.91 
	\end{bmatrix}
	\begin{bmatrix}
		x_{k,1} 
		\\x_{k,2}
		\\x_{k,3}  
	\end{bmatrix}\nonumber
	\\&+
	\begin{bmatrix}
		0.05 & 0.02 & -0.01\\ 
		0.01 & 0.05 & -0.02\\ 
		0 &0.02  & +0.02 
	\end{bmatrix}
	\begin{bmatrix}
		u_{k,1} 
		\\u_{k,2}
		\\u_{k,3}  
	\end{bmatrix} 
	-
	\begin{bmatrix}
		0.57
		\\0.37
		\\0.29 
	\end{bmatrix}
\end{eqnarray}

\hlgreen{In ML, more data collection and training does generate an accurate model, but that leads to poor generalization. In such a scenario, the model shows poor performance when presented with new data because too much training data may result in modeling the random noise, referred to as overfitting.} \hlgreen{It was found that the DTSINDYc model $ M_1 $ given by} \eqref{eq_M3}, \hlgreen{generalizes the entire dataset (black dashed-line plots in} \autoref{fig_dtsindyc_tof_states_3} (b) to (d)) \hlgreen{with a fraction of noisy training data ($m=500$ and $n=3$), which confirms the robustness of DTSINDYC modeling with low and noisy data availability. Although the model was trained with peristalsis waves of $ 8  $ mm amplitude, but it was able to generalize various amplitude  peristaltic waves as shown by the blue color plots in} \autoref{fig_dtsindyc_tof_states_3} (b) to (d).  


Likewise, model $ M_2 $ was derived from the data collected from \acp{VPS} located at $ L_5 $, $ L_6 $, and $ L_7 $. The recorded output from the \acp{VPS} exhibited noisy data in the absence of an output filter. Irrespective of the noise, \ac{DTSINDYC} robustly predicted the state with great accuracy. The algorithm was able to manage the bias-variance trade-off adequately to deliver a precise fit to the data. To emphasize the versatility of \ac{DTSINDYC}, it is important to note that the procedure followed to built and validate $M_2  $ was kept same as $M_1  $, except that the data for the respective models were collected from different sensors. 

\subsection{Stability Analysis}

A system with a bounded control input is said to be Input-to-State Stability (ISS) if each and every state trajectory stays bounded irrespective of the initial state, and  if the input grows small, the state trajectory eventually reduces \cite{Hassan_nonlinear}. Re-writing \eqref{eq_dtsindyc} to incorporate the form of diiference equations in \eqref{eq_M3} 

\begin{align}
	\label{eq_dtsindyc_rewritten}
	{\mathbf{x}}_{k+1}={\mathbf A}{\mathbf{x}}_{k}+{\mathbf {\mathbf g'}}({\mathbf{ x}}_{k})+{\mathbf{ \tilde u}}_k\text{ }| \text{ }{\mathbf{ x}}_{k}, {\mathbf{ \tilde u}}_k\in \mathbb{R}^{3}
\end{align}
In \eqref{eq_dtsindyc_rewritten}, $  {\mathbf g'}(.)$ and $ {\mathbf{ \tilde u}}_k = [ {{ \tilde u}}_{k,1}, {{ \tilde u}}_{k,2} ,  {{ \tilde u}}_{k,3}]^{T}$ represent the nonlinear, and $ 2^{nd} $ and $ 3^{rd} $ terms (input terms) of \eqref{eq_M3}, respectively. Equation \eqref{eq_dtsindyc_rewritten} has an equilibrium point at the origin. The state matrix $ \mathbf A $ determined from the linear approximation of \eqref{eq_dtsindyc_rewritten} for $ {\mathbf {\tilde u}_k} = {\mathbf 0}$ has eigenvalues at $ 0.87 $, $ 0.82 $, and $ 0.91 $, which proves that the system is asymptotically stable since all the poles are lying inside the unit circle.

By definition, a continuous function $ V $ on $  \mathbb{R}^{3}$ for \eqref{eq_dtsindyc_rewritten} becomes  ISS-Lyapunov function if it holds the following two properties 

\begin{Properties}
	\item $  \alpha_{1}(||{\bf{x}}_k||) \le$  $V({\mathbf{x}}_k)$ $ \le \alpha_{2}(||{\bf{x}}_k||)$, and 
	
	\item $\Delta V({\bf{x}}_k) = V({\bf{x}}_{k+1})-V({{\bf{x}}_{k}}) \le -\alpha_3(||{{\bf{x}}_k}||) + \sigma(||{\bf{\tilde u}}_k||)$, 
\end{Properties}
where $ \alpha_{1}(.) $, $ \alpha_2(.) $ and $ \alpha_{3}(.) $ $ \in $ class $K_{\infty} $ functions and $ \sigma(.) \in$ class $ K $ function \cite{jiang2001input}.

Let, Lyapunov function $ {V({\bf{x}}_k)} = {{\bf{x}}_k^{T}} {\mathbf P} {{\bf{x}}_k} $, where $ {\mathbf P} = diag(p_{1}, \text{ }p_{2}, \text{ }p_{3})$ is a positive definite matrix such that $ p_1 $ and $ p_2$ are in order of magnitude of $ -2 $, $ p_2<p_1 $, and $ p_3=1 $. By nature of definition of $V({\bf{x}}_k)$, it satisfies \textit{Property 1} with $ \alpha_{1}(||{\bf{x}}||_k) = \lambda{min}({\mathbf P}){|{|{\bf{x}}_k}||^{2}} $ and $ \alpha_{2}(||{\bf{x}_k}||) = \lambda{max}({\mathbf P}){|{|{\bf{x}}_k}||^{2}} $, where $ \lambda{min}({\mathbf P}) $ and $ \lambda{max}({\mathbf P}) $ are the smallest and the highest eigenvalues of $ {\mathbf P} $.  To verify \textit{Property 2}, $ \Delta{V({\bf{x}}_k)} $ is determined by using \eqref{eq_M3} and \eqref{eq_dtsindyc_rewritten}, and then it is approximated by eliminating terms based on the coefficient values. Hence, approximated $ \Delta{V({\bf{x}}_k)} $ can be written as


\begin{eqnarray}
	\label{eq_delta_V_pt2}
	\begin{split}
	\Delta{V({\bf{x}}_k)} &\approx -0.27p_1x_{1}^{2} - 0.3p_2x_{2}^{2}  - 0.17p_3{x_3}^{2} - 0.28p_2\\
	&x_2x_3	+ 0.97p_1\tilde u_1x_1 + 0.63p_2\tilde u_2x_2 - 0.12p_2{\tilde u_2}\\&x_3 +0.53p_3\tilde u_3 x_3
	+0.32p_1\tilde u_1^{2} + 0.14p_2\tilde u_2^{2}
	\end{split}
\end{eqnarray}

Using inequalities, $ -0.28p_2x_2x_3 \le 0.28p_2(x_2^{2} + x_3^{2}) $, $  \pm \tilde u_ix_j \le ||{\bf{\tilde u}}||_{\infty}|x_j|$, $ { x}_j^{2}\le||{\bf{x}}||^{2} $, and $\pm\sum c_j{\tilde u}_j^{2}\le||{\bf{\tilde u}}||^{2} $ $ \forall$ $ 0\le c_j \le 1$,  \eqref{eq_delta_V_pt2} can be re-written as


\begin{eqnarray}
	\label{eq_delta_V_pt4}
	\begin{split}
		\Delta{V({\bf{x}}_k)} &\le (-0.27p_1x_{1}^{2} + 0.97p_1||{\bf{\tilde u}}||_{\infty}|x_1| + ||{\bf{\tilde u}}||^{2})+\\& 		
		(0.63p_2||{\bf{\tilde u}}||_{\infty}|x_2|-0.02p_2x_{2}^{2})  +\{- (0.17p_3\\&-0.28p_2)x_3^{2} +(0.12p_2+0.53p_3)||{\bf{\tilde u}}||_{\infty}|x_3|)\}  
	\end{split}
\end{eqnarray}

Since, $ p3>>p2 $, \eqref{eq_delta_V_pt4} can be approximated as 

\begin{eqnarray}
	\label{eq_delta_V_pt5}
	\begin{split}
		\Delta{V({\bf{x}}_k)} &\le (-0.27p_1x_{1}^{2} + 0.97p_1||{\bf{\tilde u}}||_{\infty}|x_1| + ||{\bf{\tilde u}}||^{2}) 		
		\\&(-0.02p_2x_{2}^{2} + 0.63p_2||{\bf{\tilde u}}||_{\infty}|x_2|)  +(- 0.17p_3\\&x_3^{2} +0.53p_3||{\bf{\tilde u}}||_{\infty}|x_3|)
	\end{split}
\end{eqnarray}

For, $ {\mathbf{ \tilde u}}_k  = \bf{0}$, \eqref{eq_delta_V_pt5} reduces to $\Delta{V({\bf{x}}_k)} \le -(0.27p_1x_{1}^{2}   +0.02p_2x_{2}^{2} + 0.17p_3x_3^{2})$, which is a negative definite function. Hence, DTSINDYc model $ M_1 $ represented by \eqref{eq_M3} and \eqref{eq_dtsindyc_rewritten} is asymptotically stable under zero-input condition. To use the term $  -(0.27p_1x_{1}^{2}   +0.02p_2x_{2}^{2} + 0.17p_3x_3^{2}) $ to dominate the others terms in \eqref{eq_delta_V_pt5}, \eqref{eq_delta_V_pt5} can be re-written as



\begin{eqnarray}
	\label{eq_delta_V_pt6}
	\begin{split}
		\Delta{V({\bf{x}}_k)} &\le (1-\theta)(-0.27 p_1x_{1}^{2} -0.02 p_2x_{2}^{2} - 0.17p_3{x_3}^{2})\\&-0.27\theta p_1(x_{1}^{2} - 3.59||{\bf{\tilde u}}||_{\infty}|x_1|\theta^{-1} - 3.70||{\bf{\tilde u}}||^{2}\\&\theta^{-1})		
		-0.02 \theta p_2(x_{2}^{2} - 31.5||{\bf{\tilde u}}||_{\infty}|x_2|\theta^{-1}) \\&- 0.17\theta p_3 (x_3^{2} - 3.11||{\bf{\tilde u}}||_{\infty}|x_3|\theta^{-1})
	\end{split}
\end{eqnarray}

where $ 0<\theta<1 $. For $ |x_1|$$ > $ $\rho_1(\|{\mathbf {\tilde u}}\|_{\infty})=$ $[0.5\{(14.80\theta+12.89)\|{\mathbf {\tilde u}}\|_{\infty}^{2}\}^{0.5}+$ $1.79{\mathbf \|{\mathbf {\tilde u}}\|^{\infty}}]\theta^{-1}$, $ |x_2|>\rho_2(\|{\mathbf {\tilde u}}\|_{\infty})={31.5\|{\mathbf {\tilde u}}\|_{\infty}}{\theta}^{-1}$ , and $  |x_3|>\rho_3(\|{\mathbf {\tilde u}}\|_{\infty})={3.11\|{\mathbf {\tilde u}}\|_{\infty}}{\theta}^{-1}$, \eqref{eq_delta_V_pt6} is negative definite. In \eqref{eq_delta_V_pt6}, $ (1-\theta)(-0.27 p_1x_{1}^{2} -0.02 p_2x_{2}^{2} - 0.17p_3{x_3}^{2}) $ is a class $ K_\infty $ function and if $ \mathbf {\tilde u} $ is bounded then a class $ K $ function $ \rho(\|\mathbf {\tilde u} \|) $ can be defined such that $ \rho(\|\mathbf {\tilde u} \|) = \max \{\rho_1(\|\mathbf {\tilde u} \|)\},\rho_2(\|\mathbf {\tilde u} \|),\rho_3(\|\mathbf {\tilde u} \|) \}$ to satisfy the \textit{Property 2}. Hence, model $ M_1 $ is ISS. \textit{Note}: Without the loss of generality, in this analysis for simplifying the calculations, other values of $ |x_1| $, $ |x_2| $, and $ |x_3| $ satisfying \eqref{eq_delta_V_pt6} are not considered. 

For another variant of RoSEv2.0, the coefficients of the discovered difference equations given by \eqref{eq_M3}, but a similar kind of analysis can be performed for any set of nonlinear difference equations, provided the equations are asymptotically stable, and inputs are bounded. Intuitively, a soft pneumatic robot is an underactuated and underdamped system, which is always stable under zero actuation. If a unit-step input is applied to the RoSEv2.0, the chamber deformation always remains at the steady-state position, making the robot inherently in ISS.

%

\subsection{MPC Performance Testing}

The results presented in this section are associated with the theory and methodology presented in \ref{sec_mpc_theory} and \ref{sec_mpc_design}. Raspberry Pi 3B+ based host computer (\autoref{fig_fig_RoSEv2.0_complete_setup}) is used to run the DTSINDYc-MPC algorithm to calculate the optimal control values required for tracking the desired peristaltic waves.

\begin{figure}[!bt]
	\myfloatalign
	{\includegraphics[width=0.80\linewidth]{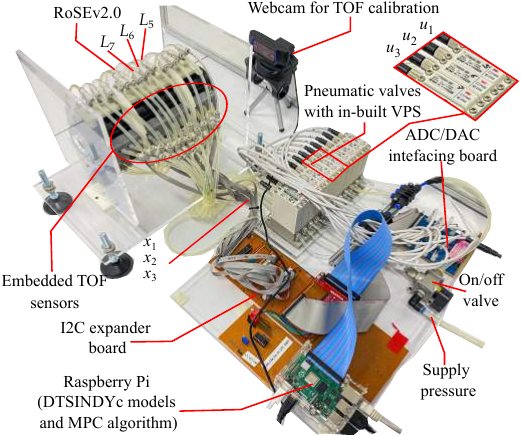}}
	\caption[RoSE complete set-up used for performing the MPC experiemnts]{RoSE complete set-up used for performing the MPC experiemnts.} 
	\label{fig_fig_RoSEv2.0_complete_setup}
\end{figure}

{For implementing the MPC, the $ \mathbf Q $, ${\mathbf{R_u}} $ and $ {\mathbf{R_{\Delta u}}} $ matrices in} \eqref{eq_mpc_control_optimization} {were chosen to be positive definite diagonal matrices. By assigning appropriate values to these matrices, the relative importance of the control goals have been considered. In this research, larger $ \mathbf Q $ values ($ \mathbf{Q} = 5\mathbf{I}_{3\times3} $, where $ \mathbf{I}_{3\times3}  $ represents identity matrix) relative to the ${\mathbf{R_u}} $ and $ {\mathbf{R_{\Delta u}}} $ values ($ {\mathbf{R_u}} = 0.5\mathbf{I}_{3\times3} $, ${\mathbf{R_{\Delta u}}}=0.5\mathbf{I}_{3\times3}$) were selected since larger ${\mathbf{R_u}} $ and $ {\mathbf{R_{\Delta u}}} $ values result in slower dynamic responses. } $ \mathbf {\tilde Q} $ was determined by solving \eqref{eq_lyapunov equation}.

The control performance of the implemented \ac{MPC} with model $ M_1 $ was estimated in terms of \ac{RMSE} (mm) and total execution time (s) taken by the \ac{MPC} algorithm to stop. Both the performance indicators were evaluated by varying the prediction horizon $N_p$ from one to eight, keeping the control horizon $ N_c =N_p$. The evaluation was done prior to the tuning of the control bounds.

To compare the \ac{MPC} performance of $N_p>1$ with $N_p=1$ (single-step ahead prediction), the \ac{RMSE} for $N_p=1$ was subtracted from \ac{RMSE} of all $N_p$. From \autoref{fig_mpc_performance} (a), it can be seen that $rmse_1$, $rmse_2$, and $rmse_3$ decrease up to $N_p=4$, where $rmse_1$, $rmse_2$, and $rmse_3$ are the root mean square error between the reference displacement signals ($ x^{(ref)}_1 $, $ x^{(ref)}_2 $, and $ x^{(ref)}_3 $) and the \ac{RoSEv2.0} \ac{TOF} outputs ($ x_1 $, $ x_2 $, and $ x_3 $) in response to the optimized control law ($  u_1$, $  u_2$, and $  u_3$). Since the \ac{RoSEv2.0} conduit radius is $ 10 $ mm so a slight increment in the error ($ >0.4 $ mm) will cause a significant impact on the \ac{MPC} performance. Until $N_p=4$ in \autoref{fig_mpc_performance} (a), all the errors tend to decrease with an increase of $N_p$. From \autoref{fig_mpc_performance} (b), it can be observed that the execution time drastically goes up as $N_p$ increments.  Since no significant difference in the \ac{MPC} performance was found for $ N_p>4 $, thus $ N_c=N_p=4 $ was chosen.

\begin{figure}[!bt]
	\myfloatalign
	{\includegraphics[width=0.8\linewidth]{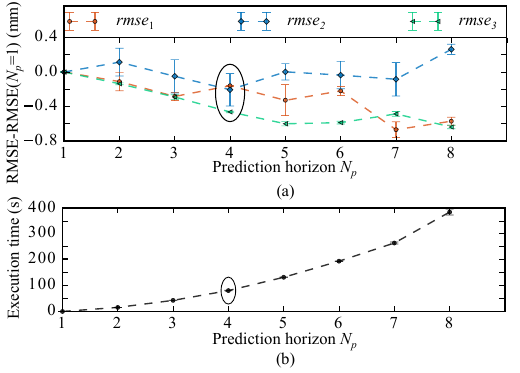}}
	\caption[Estimation of control performance of $M_1$ with different prediction horizon $N_p$.]{Estimation of control performance of $M_1$ with different prediction horizon $N_p$. (a) \ac{RMSE} between the \ac{RoSEv2.0} controlled output and \ac{MPC} reference input is plotted with $N_p$. (b) Plot showing the total execution time of \ac{MPC} with the change of $N_p$.} 
	\label{fig_mpc_performance}
\end{figure}

{Another reason for choosing $ N_p=4 $ is the generation of peristalsis waves in the range of $ 20 $ to $ 60 $ mm.s$ ^{-1} $, which is the physiological wave speed range in a human esophagus. Since the axial distance between two adjacent RoSEv2.0 chambers (located at $ L_5 $ and $ L_6 $) is $ z_d= 15 $ mm, thus to achieve $ c\le60$ mm.s$ ^{-1} $, a time difference of $ t_d\ge0.25 $ s is needed between the actuation of the chambers located in $ L_5 $ and $ L_6 $. At $ Np=4 $, an iteration time $ T_s=0.16 $ s was recorded, thus number of samples $ N_s\ge t_d/Ts=2 $ were generated. If $ Np $ increases, then $ T_s $ increases, and at $ Ts>0.25$ s, undersampling occurs, and it becomes impossible to generate a wave speed of $ 60$ mm.s$ ^{-1} $. }

After selecting a suitable prediction horizon and tuning the control bounds, the performance of \ac{MPC} with model $ M_1 $ was tested on the set-up as shown in \autoref{fig_fig_RoSEv2.0_complete_setup}. Reference time-series peristalsis waves ($ x^{(ref)}_1 $, $ x^{(ref)}_2 $, and $ x^{(ref)}_3 $) of speed $ 20 $ mm.s$ ^{-1} $, wavelength $ 75 $ mm, and amplitude $ 5 $, $ 7.5 $, and $ 10 $ mm (solid line orange plot, \autoref{fig_mpc_various_per_amplitude_tof} (a) to (c)) were applied to the \ac{MPC} with an average \ac{MPC} loop iteration time of $ T_s=0.15 $ s.

\begin{figure}[!bt]
	\myfloatalign
	{\includegraphics[width=1\linewidth]{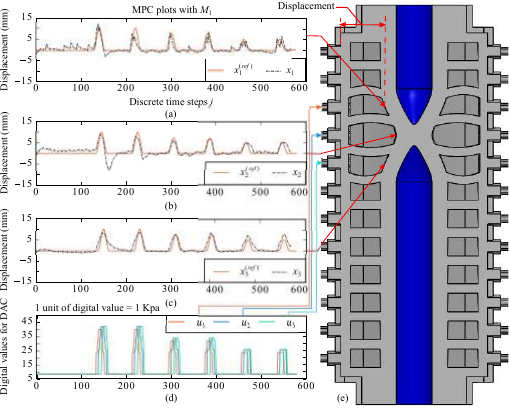}}
	\caption[\ac{MPC} performance results with model $M_1$.]{\ac{MPC} performance results with model $M_1$. The orange plots in (a) to (c) represent reference sinusoidal waveform trajectories of various amplitudes, applied to the \ac{MPC} for RoSEv2.0 $L_5$, $L_6$, and $L_7$. The dashed black line represents the response of the \acp{TOF} concerning the control applied by the \ac{MPC}, shown in (d). (e) Schematic of RoSEv2.0 illustrating the physical meaning of the recorded TOF displacement data and applied control pressure. } 
	\label{fig_mpc_various_per_amplitude_tof}
\end{figure}

In \autoref{fig_mpc_various_per_amplitude_tof} (a) to (c), it can be seen that the three \ac{RoSEv2.0} states ($ x_1 $, $ x_2 $, and $ x_3 $ corresponding to $ L_5 $, $ L_6 $, and $ L_7 $ as shown in \autoref{fig_fig_RoSEv2.0_complete_setup}) tracked the prescribed peristalsis trajectory satisfactorily when actuated with the \ac{MPC} generated control signals ($  u_1$, $  u_2$, and $  u_3$ of \autoref{fig_mpc_various_per_amplitude_tof} (d) and \autoref{fig_fig_RoSEv2.0_complete_setup}) for the pneumatic valves. In contrast to the open-loop response discussed in \ref{section_dtsindyc_models}, the implementation of \ac{RoSEv2.0} \ac{MPC} successfully determined the digital \ac{DAC} values for complete conduit occlusion, especially for layers $ L_6 $ and $ L_7 $ (\autoref{fig_dtsindyc_tof_states_3} (c) and (d)).

It can also be observed that the accuracy of the tracking is far better while inflating the chambers than deflating, which is not a matter of concern because the inflation of the chambers provides shape to the bolus tail and pushes it forward. Hence, half reference wavelength was achieved successfully. \hlcyan{Peristalsis waves of amplitude and frequency of $ 10 $ mm and $ 140  $ kHz respectively were prescribed to verify the performance of the MPC under higher frequency, which is different from the training data frequency} (\autoref{fig_mpc_tof_high_freq}). \hlcyan{The experimental results from the controller have shown an RMSE of $ 0.8 $ mm, which further reinforces the robustness capability of the MPC.  }

\begin{figure}[!bt]
	\myfloatalign
	{\includegraphics[width=0.8\linewidth]{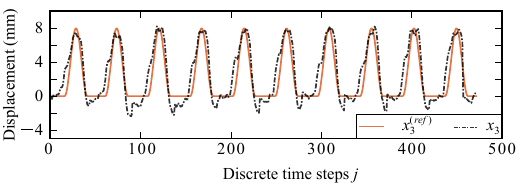}}
	\caption[MPC response under high-frequency peristaltic waves. ]{\hlcyan{MPC response under high-frequency peristaltic waves. }} 
	\label{fig_mpc_tof_high_freq}
\end{figure}

\hlgreen{To emphasize the fact that the presented control framework can be extended to any soft-robotic system with any sensor, MPC with model $ M_2 $, controlling the air chamber pressure, was implemented.} The MPC with feedback from the \acp{VPS} (\autoref{fig_fig_RoSEv2.0_complete_setup}) have shown successful tracking of the reference pressure signals (\autoref{fig_mpc_various_per_amplitude_adc}). In \autoref{fig_mpc_various_per_amplitude_adc}, only state $ x_1 $ has been shown. States  $ x_2 $, and $ x_3 $ have shown a similar satisfactorily tracking profile.

\begin{figure}[!bt]
	\myfloatalign
	{\includegraphics[width=0.8\linewidth]{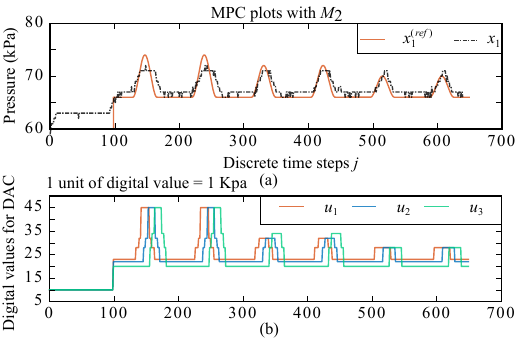}}
	\caption{\ac{MPC} performance results with model $M_2$.} 
	\label{fig_mpc_various_per_amplitude_adc}
\end{figure}

\subsection{Application of MPC in stent migration testing}

A starch-based fluid thickener (Nutulis, Nutricia, Schiphol, NL) was used to replicate the masticated boluses throughout the experimentation \cite{o2010viscosity}.
Boluses were formulated by mixing $ 1 $, $ 2 $, and $ 3 $ scoops with $ 200 $ ml of water to cover syrup, custard, and pudding-like consistencies, respectively. The associated concentrations and viscosities are 60, 80 and 100 g.L$ ^{-1} $ and 0.45 $ \pm $ 0.2,
1.2 $ \pm $ 0.4 and  3.0 $ \pm $ 1.0 Pa.s$ ^{-1} $, respectively \cite{Dirven2015}. 
After the RoSEv2.0 conduit was lubricated with artificial saliva (Aquae Dry Mouth Spray, Hamilton), the candidate stent was deployed in the  RoSEv2.0 conduit by using an intruder sheath and a retrieval thread.  The complete stent deployment setup is discussed in \cite{bhattacharya2020rose}.

Stent migration experiments with bolus swallow were performed on \ac{RoSEv2.0} by using peristaltic actuation protocol with 50 peristalsis cycles. The applied methodology for conducting the experiments has been discussed in \cite{bhattacharya2020rose}. From \autoref{fig_stent_mig_res}, it can be observed that with the increase in bolus concentration, the migration slightly increases. The obtained results are in line with the results presented in \cite{bhattacharya2020rose} in which the experiments were conducted in an open-loop manner, and the actual conduit displacement was not tracked. Since stent migration depends on the peristalsis parameters, no significant difference in migration results was found compared to the results presented in \cite{bhattacharya2020rose}. \hlgreen{The negligible migration of the stent is due to its exertion of sufficient radial force on the conduit wall, which ensures its proper fixation under peristalsis.} But, with the addition of \ac{MPC}, it can now be assured that prescribed peristalsis has been achieved because in stent implanted \ac{RoSEv2.0} experiments, the conduit displacement was controlled by the \ac{MPC} to ensure occlusion. Hence, implementing \ac{MPC} has further enhanced the application of stent migration testing in \ac{RoSEv2.0}. 

\begin{figure}[!tb]
	\myfloatalign
	{\includegraphics[width=0.8\linewidth]{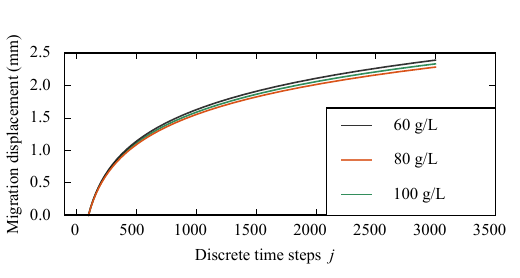}}
	\caption[Stent migration results with the change in bolus concentration, conducted for 50 peristalsis cycle.]{Stent migration results with the change in bolus concentration conducted for 50 peristalsis cycles.} 
	\label{fig_stent_mig_res}
\end{figure}

\section{Conclusion}

\ac{RoSEv2.0} was developed with embedded TOF sensors to measure its conduit deformation,
which has extended its capabilities. Additionally, with the aid of the integrated \ac{TOF} sensors, the study has first attempted to implement a closed-loop controller in the form of MPC in \ac{RoSEv2.0} (\autoref{fig_mpc_implementation}). Since, the methodology presented to design the \ac{MPC} can be extended to any other soft robotic application provided the \ac{SINDYC} model can be generated, thus MPC was also implemented to control \ac{RoSEv2.0} air chamber pressure. The procedure followed for the implementation remained the same, except for the \ac{DTSINDYC} models, which were identified on \ac{VPS} data. 
Peristalsis waves of speed 20 mm.s$ ^{-1} $, wavelength $ 75 $ mm, and amplitudes $ 5 $, $  7.5 $, and $ 10 $ mm were successfully generated by the \ac{MPC} (\autoref{fig_mpc_various_per_amplitude_tof}) for stent testing. 

Although the generated model has provided satisfactory performance for peristalsis generation and stent testing, as seen in the results, there is always some model mismatch. Generally, the model mismatch problem in an MPC is treated as a disturbance, which can be addressed by determining the Region Of Attraction (ROA). By introducing a positive invariant terminal set, the ROA can be estimated, where the MPC ensures recursive feasibility and stability. Along with the ROA determination, an online learning-based MPC in which the model gets updated over time will be designed in future studies to address the elastomer stiffness changes.

\section*{Data accessibility}

A project website for this work is available at:
\url{https://bhattner143.github.io/rosev2-dtsindyc.github.io/}.
The code used in this work is made available at:
\url{https://github.com/bhattner143/SINDYc_MPC_RoSE_symmetric_peristaltic.git}.
The data can be generated using the code.


\bibliographystyle{Bibliography/IEEEtranTIE}
\bibliography{Bibliography/IEEEabrv,Bibliography/Bibliography2}\ 

@Article{rus2015,
  author    = {Rus, Daniela and Tolley, Michael T},
  journal   = {Nature},
  title     = {Design, fabrication and control of soft robots},
  year      = {2015},
  number    = {7553},
  pages     = {467--475},
  volume    = {521},
  publisher = {Nature Publishing Group},
  doi={10.1038/nature14543},
}

@Article{Dirven2014, author={S. {Dirven} and F. {Chen} and W. {Xu} and J. E. {Bronlund} and J. {Allen} and L. K. {Cheng}}, journal={IEEE/ASME Transactions on Mechatronics}, title={Design and Characterization of a Peristaltic Actuator Inspired by Esophageal Swallowing}, year={2014}, volume={19}, number={4}, pages={1234-1242},
doi={10.1109/TMECH.2013.2276406},}

@Article{Dirven2015, author={S. {Dirven} and W. {Xu} and L. K. {Cheng}}, journal={IEEE/ASME Transactions on Mechatronics}, title={Sinusoidal Peristaltic Waves in Soft Actuator for Mimicry of Esophageal Swallowing}, year={2015}, volume={20}, number={3}, pages={1331-1337},
doi={10.1109/TMECH.2014.2337291},}

@article{Garcia2010,
	author = {Garcia, Jane M and Chambers IV, Edgar and Clark, Megan and Helverson, Jennifer and Matta, Ziad},
	title = {Quality of care issues for dysphagia: modifications involving oral fluids},
	journal = {Journal of Clinical Nursing},
	volume = {19},
	number = {11-12},
	pages = {1618-1624},
	doi = {10.1111/j.1365-2702.2009.03009.x},
	year = {2010}
}

@article{hirdes2013vitro,
	title={In vitro evaluation of the radial and axial force of self-expanding esophageal stents},
	author={Hirdes, Meike MC and Vleggaar, Frank P and De Beule, Matthieu and Siersema, Peter D},
	journal={Endoscopy},
	volume={45},
	number={12},
	pages={997--1005},
	year={2013},
	doi={10.1055/s-0033-1344985},
}

@article{sharma2010role,
	title={Role of esophageal stents in benign and malignant diseases},
	author={Sharma, Prateek and Kozarek, Richard and Practice Parameters Committee of the American College of Gastroenterology and others},
	journal={American Journal of Gastroenterology},
	volume={105},
	number={2},
	pages={258--273},
	year={2010},
	doi={10.1038/ajg.2009.684},
}

@article{hanawa2009materials,
	title={Materials for metallic stents},
	author={Hanawa, Takao},
	journal={Journal of Artificial Organs},
	volume={12},
	number={2},
	pages={73--79},
	year={2009},
	publisher={Springer},
	doi={10.1007/s10047-008-0456-x},
}

@article{garbey2016esophageal,
	title={Esophageal stent migration: Testing few hypothesis with a simplified mathematical model},
	author={Garbey, Marc and Salmon, Remi and Fikfak, Vid and Clerc, Claude O},
	journal={Computers in biology and medicine},
	volume={79},
	pages={259--265},
	year={2016},
	publisher={Elsevier},
	doi={10.1016/j.compbiomed.2016.10.024}
}

@article{mozafari2018migration,
	title={Migration resistance of esophageal stents: The role of stent design},
	author={Mozafari, Hozhabr and Dong, Pengfei and Zhao, Shijia and Bi, Yonghua and Han, Xinwei and Gu, Linxia},
	journal={Computers in biology and medicine},
	volume={100},
	pages={43--49},
	year={2018},
	publisher={Elsevier},
	doi={10.1016/j.compbiomed.2018.06.031}
}

@article {brunton2016discovering,
	author = {Brunton, Steven L. and Proctor, Joshua L. and Kutz, J. Nathan},
	title = {Discovering governing equations from data by sparse identification of nonlinear dynamical systems},
	volume = {113},
	number = {15},
	pages = {3932--3937},
	year = {2016},
	doi = {10.1073/pnas.1517384113},
	publisher = {National Academy of Sciences},
	issn = {0027-8424},
	journal = {Proceedings of the National Academy of Sciences}
}

@article{brunton2016extracting,
	author={Brunton, Bingni W and Johnson, Lise A and Ojemann, Jeffrey G and Kutz, J Nathan},
	title = {Extracting spatial-temporal coherent patterns in large-scale neural recordings using dynamic mode decomposition},
	journal = {Journal of Neuroscience Methods},
	volume = {258},
	pages = {1 - 15},
	year = {2016},
	issn = {0165-0270},
	doi = {10.1016/j.jneumeth.2015.10.010}
}

@article{brunton2016sparse,
	author = {Brunton, Steven L. and Proctor, Joshua L. and Kutz, J. Nathan},
	title = {Sparse Identification of Nonlinear Dynamics with Control ({SINDY}c)},
	journal = {IFAC-PapersOnLine},
	volume = {49},
	number = {18},
	pages = {710-715},
	ISSN = {2405-8963},
	DOI = {10.1016/j.ifacol.2016.10.249},
	year = {2016},
	type = {Journal Article}
}

@book{dullerud2013course,
	title={A course in robust control theory: a convex approach},
	author={Dullerud, Geir E and Paganini, Fernando},
	volume={36},
	year={2013},
	publisher={Springer Science \& Business Media}
}

@article{mayne2000constrained,
	title={Constrained model predictive control: Stability and optimality},
	author={Mayne, David Q and Rawlings, James B and Rao, Christopher V and Scokaert, Pierre OM},
	journal={Automatica},
	volume={36},
	number={6},
	pages={789--814},
	year={2000},
	publisher={Elsevier},
	doi={10.1016/S0005-1098(99)00214-9}
}

@article{clarke1987generalizeda,
	title={Generalized predictive control-Part I. {T}he basic algorithm},
	author={Clarke, David W and Mohtadi, Coorous and Tuffs, PS},
	journal={Automatica},
	volume={23},
	number={2},
	pages={137--148},
	year={1987},
	publisher={Elsevier},
	doi={10.1016/0005-1098(87)90087-2}
}

@article{garcia1989model,
	title={Model predictive control: theory and practice-a survey},
	author={Garcia, Carlos E and Prett, David M and Morari, Manfred},
	journal={Automatica},
	volume={25},
	number={3},
	pages={335--348},
	year={1989},
	doi={10.1016/0005-1098(89)90002-2}
}

@article{golbert2004model,
	title={Model-based control of fuel cells::(1) regulatory control},
	author={Golbert, Joshua and Lewin, Daniel R},
	journal={Journal of power sources},
	volume={135},
	number={1-2},
	pages={135--151},
	year={2004},
	publisher={Elsevier},
	doi={10.1016/j.jpowsour.2004.04.008}
}

@article{andersen1992evaluating,
	title={Evaluating estimation of gain directionality: Part 1: Methodology},
	author={Andersen, Henrik Weisberg and K{\"u}mmel, Mogens},
	journal={Journal of Process Control},
	volume={2},
	number={2},
	pages={59--66},
	year={1992},
	publisher={Elsevier},
	doi={10.1016/0959-1524(92)80002-F}
}

@article{kaiser2018sparse,
	title={Sparse identification of nonlinear dynamics for model predictive control in the low-data limit},
	author={Kaiser, Eurika and Kutz, J Nathan and Brunton, Steven L},
	journal={Proceedings of the Royal Society A},
	volume={474},
	number={2219},
	pages={20180335},
	year={2018},
	publisher={The Royal Society Publishing},
	doi={10.1098/rspa.2018.0335},
}

@article{bhattacharya2019sparse,
	title={Sparse machine learning discovery of dynamic differential equation of an esophageal swallowing robot},
	author={Bhattacharya, Dipankar and Cheng, Leo K and Xu, Weiliang},
	journal={IEEE Transactions on Industrial Electronics},
	volume={67},
	number={6},
	pages={4711--4720},
	year={2019},
	publisher={IEEE},
	doi={10.1109/TIE.2019.2928239}
}

@article{bhattacharya2020rose,
	author = {Bhattacharya, Dipankar and J.V. Ali, Sherine and Cheng, Leo K. and Xu, Weiliang},
	title = {Ro{SE}: A Robotic Soft Esophagus for Endoprosthetic Stent Testing},
	journal = {Soft Robotics},
	doi = {10.1089/soro.2019.0205}
}

@article{misra2001mathematical,
	title={A mathematical model for oesophageal swallowing of a food-bolus},
	author={Misra, JC and Pandey, SK},
	journal={Mathematical and Computer Modelling},
	volume={33},
	number={8-9},
	pages={997--1009},
	year={2001},
	publisher={Elsevier},
	doi={10.1016/S0895-7177(00)00295-8Get}
}

@article{takagi2011peristaltic,
	title={Peristaltic pumping of viscous fluid in an elastic tube},
	author={Takagi, D and Balmforth, NJ},
	journal={Journal of Fluid Mechanics},
	volume={672},
	pages={196},
	year={2011},
	publisher={Cambridge University Press},
	doi={10.1017/S0022112010005914}
}

@article{o2010viscosity,
	title={Viscosity and non-Newtonian features of thickened fluids used for dysphagia therapy},
	author={O'Leary, Mark and Hanson, Ben and Smith, Christina},
	journal={Journal of Food Science},
	volume={75},
	number={6},
	pages={E330--E338},
	year={2010},
	publisher={Wiley Online Library},
	doi={0.1111/j.1750-3841.2010.01673.x}
}

@INPROCEEDINGS{hou2020underwater,

	author={X. {Hou} and S. {Guo} and L. {Shi} and H. {Xing} and Z. {Li} and D. {Xia} and M. {Zhou} and Y. {Liu}},

	booktitle={2020 IEEE International Conference on Mechatronics and Automation (ICMA)}, 

	title={Underwater Obstacle Avoiding Trajectory Tracking Approach for Amphibious Spherical Robots}, 

	year={2020},

	volume={},

	number={},

	pages={1348-1353},

	doi={10.1109/ICMA49215.2020.9233669}}

@article{dang2020sogut,
	author = {Dang, Yu and Liu, Yuanxiang and Hashem, Ryman and Bhattacharya, Dipankar and Allen, Jacqueline and Stommel, Martin and Cheng, Leo K. and Xu, Weiliang},
	title = {SoGut: A Soft Robotic Gastric Simulator},
	journal = {Soft Robotics},
	doi={10.1089/soro.2019.0136}
	}

@article{desilva2020pysindy,
	doi = {10.21105/joss.02104},
	year = {2020},
	publisher = {The Open Journal},
	volume = {5},
	number = {49},
	pages = {2104},
	author = {Brian de Silva and Kathleen Champion and Markus Quade and Jean-Christophe Loiseau and J. Kutz and Steven Brunton},
	title = {Py{SIND}y: A Python package for the sparse identification of nonlinear dynamical systems from data},
	journal = {Journal of Open Source Software}
}

@article{li2017model,
	title={Model-free control for continuum robots based on an adaptive Kalman filter},
	author={Li, Minhan and Kang, Rongjie and Branson, David T and Dai, Jian S},
	journal={IEEE/ASME Transactions on Mechatronics},
	volume={23},
	number={1},
	pages={286--297},
	year={2017},
	publisher={IEEE},
	doi={10.1109/TMECH.2017.2775663}
}

@article{lee2017nonparametric,
	title={Nonparametric online learning control for soft continuum robot: An enabling technique for effective endoscopic navigation},
	author={Lee, Kit-Hang and Fu, Denny KC and Leong, Martin CW and Chow, Marco and Fu, Hing-Choi and Althoefer, Kaspar and Sze, Kam Yim and Yeung, Chung-Kwong and Kwok, Ka-Wai},
	journal={Soft robotics},
	volume={4},
	number={4},
	pages={324--337},
	year={2017},
	publisher={Mary Ann Liebert, Inc. 140 Huguenot Street, 3rd Floor New Rochelle, NY 10801 USA},
	doi={10.1089/soro.2016.0065
}}

@article{braganza2007neural,
	title={A neural network controller for continuum robots},
	author={Braganza, David and Dawson, Darren M and Walker, Ian D and Nath, Nitendra},
	journal={IEEE transactions on robotics},
	volume={23},
	number={6},
	pages={1270--1277},
	year={2007},
	publisher={IEEE},
	doi={0.1109/TRO.2007.906248}
}

@article{best2016new,
	title={A new soft robot control method: Using model predictive control for a pneumatically actuated humanoid},
	author={Best, Charles M and Gillespie, Morgan T and Hyatt, Phillip and Rupert, Levi and Sherrod, Vallan and Killpack, Marc D},
	journal={IEEE Robotics \& Automation Magazine},
	volume={23},
	number={3},
	pages={75--84},
	year={2016},
	publisher={IEEE},
	doi={10.1109/MRA.2016.2580591}
}

@article{george2018control,
	title={Control strategies for soft robotic manipulators: A survey},
	author={George Thuruthel, Thomas and Ansari, Yasmin and Falotico, Egidio and Laschi, Cecilia},
	journal={Soft robotics},
	volume={5},
	number={2},
	pages={149--163},
	year={2018},
	publisher={Mary Ann Liebert, Inc. 140 Huguenot Street, 3rd Floor New Rochelle, NY 10801 USA},
	doi={10.1089/soro.2017.0007}
}

@inbook{sahi2004,
	title={Digital Control Engineering: Analysis and Design},
	chapter={11},
	author={Fadali, M. Sami},
	year={2020},
	pages={450--461},
	publisher={London, United Kingdom : Academic Press, an imprint of Elsevier},
	Volume = {Third edition}
}

@book{Hassan_nonlinear,
	publisher = {New York : Toronto : New York: Macmillan Pub. Co. ; Maxwell Macmillan Canada ; Maxwell Macmillan International},
	isbn = {002363541X},
	year = {1992},
	title = {Nonlinear systems},
	language = {eng},
	address = {New York : Toronto : New York},
	author = {Khalil, Hassan K.},
}

@inbook{Maciejowski_Predictive,
	publisher = {Prentice Hall},
	year = {2002},
	chapter = {6},
	pages = {167--180},
	title = {Predictive control : with constraints},
	language = {eng},
	address = {Harlow, England ; New York},
	author = {Maciejowski, Jan Marian},
	lccn = {00054713},
}

@article{jiang2001input,
	title={Input-to-state stability for discrete-time nonlinear systems},
	author={Jiang, Zhong-Ping and Wang, Yuan},
	journal={Automatica},
	volume={37},
	number={6},
	pages={857--869},
	year={2001},
	publisher={Elsevier},
	doi={10.1016/S0005-1098(01)00028-0}
}
\begin{IEEEbiography}[{\includegraphics[width=1in,height=1.25in,clip,keepaspectratio]{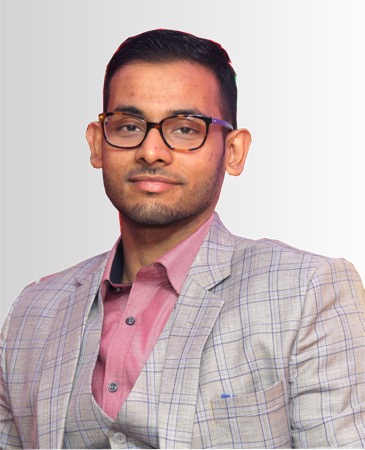}}]
{Dipankar Bhattacharya} Dipankar Bhattacharya received the B.Tech.
degree in electronics and communication engineering
from the NERIST Itanagar,
Arunachal Pradesh, India, in 2010, M.Tech. degree in electrical engineering from
the IIT Roorkee,
Roorkee, India, in 2013,  and Ph.D. degree in mechatronics engineering from the University of Auckland, New Zealand in 2021. He is currently working as a Postdoctoral Research Fellow in mechanical and automation engineering in the Chinese University of Hong Kong, Hong Kong. His research interests include biologically
inspired soft robots, cable-driven parallel robots and machine learning.
\end{IEEEbiography}

\begin{IEEEbiography}[{\includegraphics[width=1in,height=1.25in,clip,keepaspectratio]{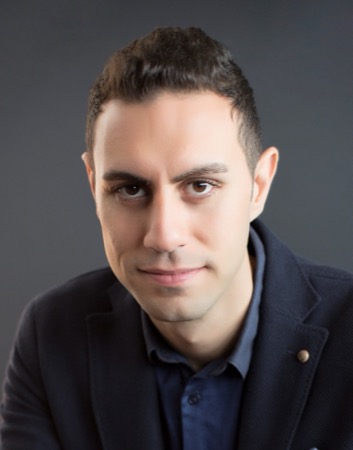}}]{Ryman Hashem}
	received the B.E. degree in mechatronics engineering from AMA International University, Salmabad, Bahrain, in 2012. He received his M.E. (First Class Hons.) and Ph.D. degree in mechatronics engineering from the University of Auckland, Auckland, New Zealand, in 2015, and 2021, respectively. He is currently working as a Postdoctoral Research Fellow in Bio-Inspired Robotics Laboratory (BIRL) in the department of engineering from the University of Cambridge, United Kingdom. His current research interests include biologically-inspired soft robots design and control.
\end{IEEEbiography}

\begin{IEEEbiography}[{\includegraphics[width=1in,height=1.25in,clip,keepaspectratio]{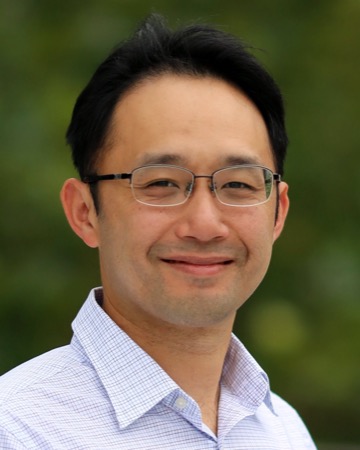}}]
{Leo K. Cheng} received the B.E. (Hons.) degree
in engineering science and the Ph.D. degree in
bioengineering from the University of Auckland,
Auckland, New Zealand, in 1997 and 2002, respectively.
He is currently a Professor with Auckland
Bioengineering Institute, Auckland. His main research
interests include the understanding of
electrophysiological events in the gastrointestinal
tract and the heart, and their relationship to
mechanical function.
\end{IEEEbiography}

\begin{IEEEbiography}[{\includegraphics[width=1in,height=1.25in,clip,keepaspectratio]{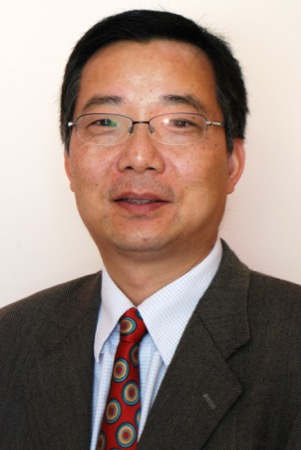}}]
{Weiliang Xu} received the B.E. degree
in manufacturing engineering and the M.E. degree
in mechanical engineering from Southeast
University, Nanjing, China, in 1982 and 1985, respectively,
and the Ph.D. degree in mechatronics
and robotics from Beihang University, Beijing, China, in 1988.
He joined the University of Auckland, Auckland,
New Zealand, in 2011, as the Chair in
Mechatronics Engineering. His current research
interests are mainly advanced mechatronics and robotics
with applications in medicine and foods. 
\end{IEEEbiography}

\end{document}